\documentclass{article} 
\usepackage{iclr2023_conference,times}

\usepackage{amsmath,amsfonts,bm}

\def\eqref#1{equation~\ref{#1}}

\def\1{\bm{1}}

\DeclareMathAlphabet{\mathsfit}{\encodingdefault}{\sfdefault}{m}{sl}
\SetMathAlphabet{\mathsfit}{bold}{\encodingdefault}{\sfdefault}{bx}{n}

\usepackage{footmisc}

\usepackage[colorlinks=true,
            linkcolor=violet,
            citecolor=violet,
            anchorcolor=violet,
            urlcolor=violet]{hyperref}
\usepackage{url}            

\usepackage{booktabs}       
\usepackage{amsfonts}       
\usepackage{nicefrac}       
\usepackage{microtype}      
\usepackage{xcolor}         
\usepackage{bm}             
\usepackage{graphicx}       
\usepackage{amsmath}        
\usepackage{amssymb}        
\usepackage{amsthm}         
\usepackage{xspace}         
\usepackage{adjustbox}      
\usepackage{subfigure}      
\usepackage{multirow}       
\usepackage{pifont}         
\usepackage[ruled,vlined]{algorithm2e}  
\newcommand{\itm}[1]{\textrm{\textit{#1}}}                  
\newcommand{\boldz}{\bm{z}}                                 
\newcommand{\boldx}{\bm{x}}                                 
\newcommand{\boldeps}{\bm{\epsilon}}                        
\newcommand{\boldsigma}{\bm{\sigma}}                        
\newcommand{\boldtheta}{\bm{\theta}}                        
\newcommand{\boldmu}{\bm{\mu}}                              
\newcommand{\gen}{G}                                        
\usepackage{eqparbox}                                       
\newcommand{\up}{$\uparrow$}                                
\newcommand{\down}{$\downarrow$}                            
\newcommand{\psnr}{PSNR}                                    
\newcommand{\dalle}[0]{DALL$\cdot$E\xspace}                 
\newcommand{\cyclediff}[0]{CycleDiffusion\xspace}           

\definecolor{mdgreen}{rgb}{0.05,0.6,0.05}
\definecolor{mdblue}{rgb}{0,0,0.7}
\definecolor{dkblue}{rgb}{0,0,0.5}
\definecolor{dkgray}{rgb}{0.3,0.3,0.3}
\definecolor{slate}{rgb}{0.25,0.25,0.4}
\definecolor{gray}{rgb}{0.5,0.5,0.5}
\definecolor{ltgray}{rgb}{0.7,0.7,0.7}
\definecolor{lavender}{rgb}{0.65,0.55,1.0}

\definecolor{mypurple}{RGB}{111,61,121}
\definecolor{myred}{RGB}{181,68,106}

\definecolor{hanblue}{rgb}{0.27, 0.42, 0.81}                
\newcommand{\ddpmcolor}[1]{\textcolor{red}{#1}}          
\newcommand{\ddimcolor}[1]{\textcolor{brown}{#1}}          
\newcommand{\latentcolor}[1]{\textcolor{hanblue}{#1}}       
\newcommand{\imagecolor}[1]{\textcolor{purple}{#1}}         
\newcommand{\controlcolor}[1]{\textcolor{orange}{#1}}       

\newcommand{\baby}[0]{
$\big\{$\raisebox{-.4em}{\includegraphics[width=.03\textwidth]{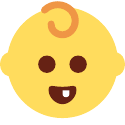}}\raisebox{-.4em}{\includegraphics[width=.03\textwidth]{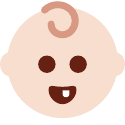}} $\cdots\big\}$}  
\newcommand{\oldperson}[0]{
$\big\{$\raisebox{-.4em}{\includegraphics[width=.03\textwidth]{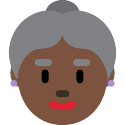}}\raisebox{-.4em}{\includegraphics[width=.03\textwidth]{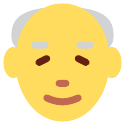}} $\cdots\big\}$}  
\newcommand{\eyeglasses}[0]{
\raisebox{-.45em}{\includegraphics[width=.035\textwidth]{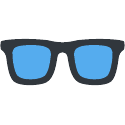}}\xspace}  
\newcommand{\yellowhat}[0]{
\raisebox{-.45em}{\includegraphics[width=.035\textwidth]{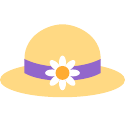}}\xspace}  

\renewcommand{\cite}[1]{\citep{#1}}                         

\newcommand{\anonymoustext}[1]{}                          
\newcommand{\acceptedtext}[1]{#1}                             
\newcommand{\opensourceunify}{\url{https://github.com/ChenWu98/unified-generative-zoo}}   
\newcommand{\opensourcezero}{\url{https://github.com/ChenWu98/cycle-diffusion}}   

\usepackage{inconsolata}

\title{
Unifying Diffusion Models' Latent Space, with Applications to \cyclediff and Guidance 
}

\author{Chen Henry Wu, \ \ Fernando De la Torre \\
Robotics Institute, Carnegie Mellon University \\
\texttt{\{chenwu2,ftorre\}@cs.cmu.edu} \\
}

\iclrfinalcopy 

\begin{document}

\maketitle

\begin{abstract}
    Diffusion models have achieved unprecedented performance in generative modeling. The commonly-adopted formulation of the latent code of diffusion models is a sequence of gradually denoised samples, as opposed to the simpler (e.g., Gaussian) latent space of GANs, VAEs, and normalizing flows. This paper provides an alternative, Gaussian formulation of the latent space of diffusion models, as well as a reconstructable \textbf{DPM-Encoder} that maps images into the latent space. While our formulation is purely based on the definition of diffusion models, we demonstrate several intriguing consequences. (1) Empirically, we observe that a common latent space emerges from two diffusion models trained independently on related domains. In light of this finding, we propose \textbf{\cyclediff}, which uses DPM-Encoder for unpaired image-to-image translation. Furthermore, applying \cyclediff to text-to-image diffusion models, we show that large-scale text-to-image diffusion models can be used as \textbf{zero-shot} image-to-image editors. (2) One can guide pre-trained diffusion models and GANs by controlling the latent codes in a unified, plug-and-play formulation based on energy-based models. Using the CLIP model and a face recognition model as guidance, we demonstrate that diffusion models have better coverage of low-density sub-populations and individuals than GANs.\footnote{\anonymoustext{Our codes will be publicly available.}\acceptedtext{The code is publicly available at \url{https://github.com/ChenWu98/cycle-diffusion}.}} 
\end{abstract}

\section{Introduction}

Diffusion models~\cite{SongE19,HoJA20} have achieved unprecedented results in generative modeling and are instrumental to text-to-image models such as \dalle 2 \cite{Ramesh2022HierarchicalTI}. Unlike GANs~\cite{Goodfellow2014GenerativeAN}, VAEs \cite{KingmaW13}, and normalizing flows \cite{Dinh2014NICE}, which have a simple (e.g., Gaussian) latent space, the commonly-adopted formulation of the ``latent code'' of diffusion models is a sequence of gradually denoised images. This formulation makes the prior distribution of the ``latent code'' data-dependent, deviating from the idea that generative models are mappings from simple noises to data \cite{Goodfellow2014GenerativeAN}. 

This paper provides a unified view of generative models of images
by reformulating various diffusion models as deterministic maps from a Gaussian latent code $\latentcolor{\boldz}$ to an image $\imagecolor{\boldx}$ (Figure~\ref{fig:conversion}, Section~\ref{subsec:unify}). 
A question that follows is \textit{encoding}: how to map an image $\imagecolor{\boldx}$ to a latent code $\latentcolor{\boldz}$. Encoding has been studied for many generative models. For instance, VAEs and normalizing flows have encoders by design, GAN inversion \cite{Xia2021GANIA} builds \textit{post hoc} encoders for GANs, and deterministic diffusion probabilistic models (DPMs) \cite{song2021denoising,0011SKKEP21} build encoders with forward ODEs. However, it is still unclear how to build an encoder for stochastic DPMs such as DDPM \cite{HoJA20}, non-deterministic DDIM \cite{song2021denoising}, and latent diffusion models \cite{Rombach2021ldm}. We propose \textbf{DPM-Encoder} (Section~\ref{subsec:encoding}), a reconstructable encoder for stochastic DPMs. 

We show that some intriguing consequences emerge from our definition of the latent space of diffusion models and our DPM-Encoder. First, observations have been made that, given two diffusion models, a fixed ``random seed'' produces similar images \cite{Nichol2021GLIDETP}. 
Under our formulation, we formalize ``similar images'' via an upper bound of image distances. Since the defined latent code contains all randomness during sampling, DPM-Encoder is similar-in-spirit to inferring the ``random seed'' from real images. Based on this intuition and the upper bound of image distances, we propose \textbf{\cyclediff} (Section~\ref{subsec:cyclediff-method}), a method for unpaired image-to-image translation using our DPM-Encoder. 
Like the GAN-based UNIT method \cite{LiuBK17}, \cyclediff encodes and decodes images using the common latent space. Our experiments show that \cyclediff outperforms previous methods based on GANs or diffusion models (Section~\ref{subsec:align}). Furthermore, by applying large-scale text-to-image diffusion models (e.g., Stable Diffusion; \citealp{Rombach2021ldm}) to \cyclediff, we obtain \textbf{zero-shot} image-to-image editors (Section~\ref{subsec:dalle}). 

\begin{figure}[!t]
\centering
  \includegraphics[width=0.98\linewidth]{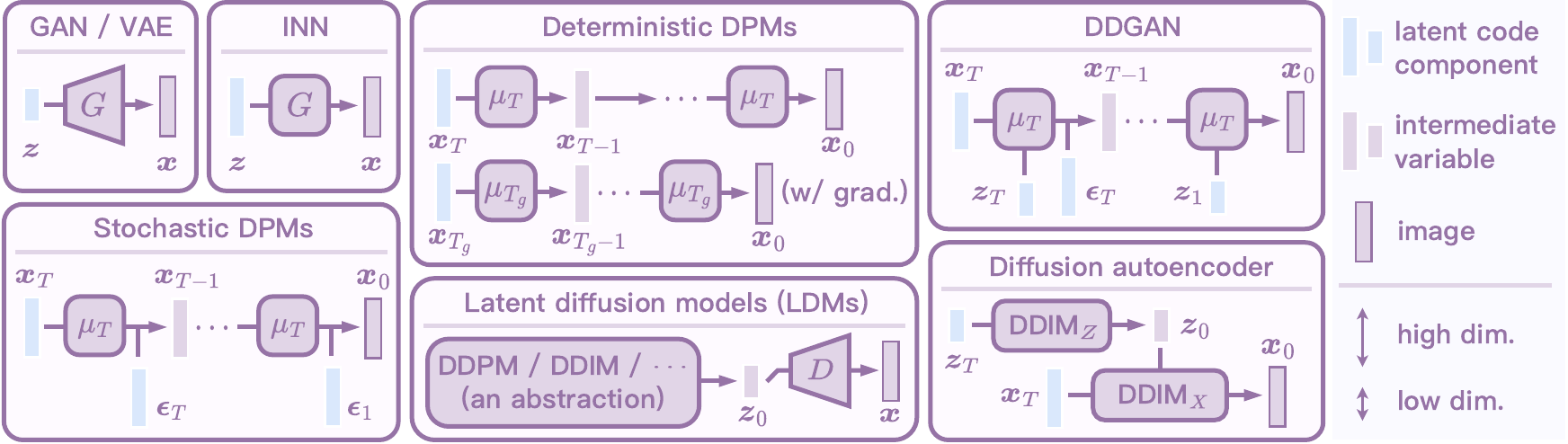}
\caption{\label{fig:conversion} Once trained, various types of diffusion models can be reformulated as deterministic maps from latent code $\latentcolor{\boldz}$ to image $\imagecolor{\boldx}$, like GANs, VAEs, and normalizing flows. }
\end{figure}

With a simple latent prior, generative models can be guided in a plug-and-play manner by means of energy-based models \cite{Nguyen2017PlugP,nie2021controllable,promptgen}. Thus, our unification allows unified, plug-and-play guidance for various diffusion models and GANs (Section~\ref{subsec:control}), which avoids finetuning the guidance model on noisy images for diffusion models \cite{dhariwal2021diffusion,Liu2021MoreCF}. 
With the CLIP model and a face recognition model as guidance, we show that diffusion models have broader coverage of low-density sub-populations and individuals (Section~\ref{subsec:control-experiment}).

\section{Related Work}
\label{sec:related-work}

Recent years have witnessed a great progress in generative models, such as GANs \cite{Goodfellow2014GenerativeAN}, diffusion models \cite{SongE19,HoJA20,dhariwal2021diffusion}, VAEs \cite{KingmaW13}, normalizing flows \cite{Dinh2014NICE}, and their hybrid extensions \cite{Sinha2021D2CDM,Vahdat2021ScorebasedGM,Zhang2021DiffusionNF,Kim2022MaximumLT}. 
Previous works have shown that their training objectives are related, e.g., diffusion models as VAEs \cite{HoJA20,Kingma2021VariationalDM,Huang2021AVP}; GANs and VAEs as KL divergences \cite{hu2018on} or mutual information with consistency constraints \cite{Zhao2018TheIF}; a recent attempt \cite{Zhang2022UnifyingGM} has been made to unify several generative models as GFlowNets \cite{Bengio2021GFlowNetF}. In contrast, this paper unifies generative models as deterministic mappings from Gaussian noises to data (\textit{aka} implicit models) once they are trained. Generative models with non-Gaussian randomness \cite{Davidson2018HypersphericalVA,Nachmani2021NonGD} can be unified as deterministic mappings in similar ways.

One of the most fundamental challenges in generative modeling is to design an encoder that is both computationally efficient and invertible. GAN inversion trains an encoder after GANs are pre-trained \cite{Xia2021GANIA}. VAEs and normalizing flows have their encoders by design. \citet{song2021denoising,0011SKKEP21} studied encoding for ODE-based deterministic diffusion probabilistic models (DPMs). However, it remains unclear how to encode for general stochastic DPMs, and DPM-Encoder fills this gap. 
Also, \cyclediff can be seen as an extension of \citet{Su2022DualDI}'s DDIB approach to stochastic DPMs. 

Previous works have formulated plug-and-play guidance of generative models as latent-space energy-based models (EBMs) \cite{Nguyen2017PlugP,nie2021controllable,promptgen}, and our unification makes it applicable to various diffusion models, which are effective for modeling images, audio \cite{Kong2021DiffWaveAV}, videos \cite{Ho2022VideoDM,Hoppe2022DiffusionMF}, molecules \cite{xu2022geodiff}, 3D objects \cite{Luo2021DiffusionPM}, and text \cite{Li2022DiffusionLMIC}.
This plug-and-play guidance can provide principled, fine-grained model comparisons of coverage of sub-populations and individuals on the same dataset. 

A concurrent work observed that fixing both (1) the random seed and (2) the cross-attention map in Transformer-based text-to-image diffusion models results in images with minimal changes \cite{Hertz2022PrompttoPromptIE}. The idea of fixing the cross-attention map is named Cross Attention Control (CAC) in that work, which can be used to edit \textit{model-generated} images \textit{when the random seed is known}. For \textit{real} images with stochastic DPMs, they generate masks based on the attention map because the random seed is unknown for real images. In Section~\ref{subsec:dalle}, we show that \cyclediff and CAC can be combined to improve the structural preservation of image editing.

\begin{table}[!th]
	\centering
	\caption{Details of redefining various diffusion models' latent space (Section~\ref{subsec:unify}).
	}\smallskip
	\label{tab:diffusion-unify}
	\begin{adjustbox}{width=\linewidth}
		\begin{tabular}{@{}lcc@{}}
			\toprule
			& Latent code $\latentcolor{\boldz}$ & Deterministic map $\imagecolor{\boldx} = \gen(\latentcolor{\boldz})$ \\
			\midrule
			\multirow{2}*{Stochastic DPMs} & \multirow{2}*{$\latentcolor{\boldz := \big(\boldx_T \oplus \boldeps_T \oplus \cdots \oplus \boldeps_{1}\big)}$} & $\boldx_{T-1} = \boldmu_{T}(\latentcolor{\boldx_{T}}, T) + \boldsigma_T \odot \latentcolor{\boldeps_T},$ \\
			& & $\boldx_{t-1} = \boldmu_{T}(\boldx_{t}, t) + \boldsigma_t \odot \latentcolor{\boldeps_t} \ (t < T), \quad \imagecolor{\boldx := \boldx_0}$. \\
			\midrule
			\multirow{2}*{Deterministic DPMs} & \multirow{2}*{$\latentcolor{\boldz := \boldx_{T}}$ ($T = T_{\textit{g}}$ if with gradient) } & $\boldx_{T-1} = \boldmu_{T}(\latentcolor{\boldx_{T}}, T), $ \\
			& & $\boldx_{t-1} = \boldmu_{T}(\boldx_{t}, t) \ (t < T), \quad \imagecolor{\boldx := \boldx_0}. $ \\
			\midrule
            LDM    & $\latentcolor{\boldz}$ of $\gen_{\text{latent}}$ & $\boldz_0 = \gen_{\text{latent}}(\latentcolor{\boldz}), \quad \imagecolor{\boldx} = D(\boldz_0)$. \\
            \midrule
            DiffAE & $\latentcolor{\boldz := \big(\boldz_T \oplus \boldx_T\big)}$ & $\boldz_0 = \text{DDIM}_{Z} (\latentcolor{\boldz_T}), \quad \imagecolor{\boldx := \boldx_0} = \text{DDIM}_{X} (\latentcolor{\boldx_T}, \boldz_0)$. \\
            \midrule
            \multirow{3}*{DDGAN}  & \multirow{2}*{$\latentcolor{\boldz := \big(\boldx_T \oplus \boldz_T \oplus \boldeps_T \oplus \cdots}$} & $\boldx_{T-1} = \boldmu_{T}(\latentcolor{\boldx_{T}}, \latentcolor{\boldz_T}, T) + \boldsigma_T \odot \latentcolor{\boldeps_{T}}, $  \\
            & \multirow{2}*{$\latentcolor{\oplus \ \boldz_{2} \oplus \boldeps_{2} \oplus \boldz_{1}\big)}$} & $\boldx_{t-1} = \boldmu_{T}(\boldx_{t}, \latentcolor{\boldz_t}, t) + \boldsigma_t \odot \latentcolor{\boldeps_{t}} \ (1 < t < T), $ \\
            & & $\imagecolor{\boldx := \boldx_{0}} = \boldmu_{T}(\boldx_{1}, \latentcolor{\boldz_1}, 1).$ \\
			\bottomrule
		\end{tabular}
	\end{adjustbox}
\end{table}

\section{Method}
\label{sec:method}

\subsection{Gaussian Latent Space for Diffusion Models}
\label{subsec:unify}

Generative models such as GANs, VAEs, and normalizing flows can be seen as a family of \textit{implicit models}, meaning that they are deterministic maps $\gen: \mathbb{R}^{d} \to \mathcal{X}$ from latent codes $\latentcolor{\boldz}$ to images $\imagecolor{\boldx}$. At inference, sampling from the image prior $\imagecolor{\boldx} \sim p_{\imagecolor{\boldx}}(\imagecolor{\boldx})$ is implicitly defined as $\latentcolor{\boldz} \sim p_{\latentcolor{\boldz}}(\latentcolor{\boldz}), \imagecolor{\boldx} = \gen(\latentcolor{\boldz})$. The latent prior $p_{\latentcolor{\boldz}}(\latentcolor{\boldz})$ is commonly chosen to be the isometric Gaussian distribution. In this section, we show how to unify diffusion models into this family. Overview is shown in Figure~\ref{fig:conversion} and Table~\ref{tab:diffusion-unify}. 

\noindent\textbf{Stochastic DPMs:} Stochastic DPMs \cite{HoJA20,SongE19,0011SKKEP21,song2021denoising,Watson2022LearningFS} generate images with a Markov chain. Given the mean estimator $\boldmu_{T}$ (see Appendix~\ref{app:math-details}) and $\latentcolor{\boldx_{T}} \sim \mathcal{N}(\bm{0}, \bm{I})$, the image \imagecolor{$\boldx := \boldx_0$} is generated through $\boldx_{t-1} \sim \mathcal{N}(\boldmu_{T}(\boldx_{t}, t), \mathrm{diag}(\boldsigma_t^2))$.
Using the reparameterization trick, we define the latent code $\latentcolor{\boldz}$ and the mapping $\gen$ recursively as
\begin{equation}
\label{eq:diffusion-as-implicit-latent}
\begin{split}
    &\latentcolor{\boldz := \big(\boldx_T \oplus \boldeps_T \oplus \cdots \oplus \boldeps_{1}\big)} \sim \mathcal{N}(\bm{0}, \bm{I}), \ \ \ \boldx_{t-1} = \boldmu_{T}(\boldx_{t}, t) + \boldsigma_t \odot \latentcolor{\boldeps_t}, \ \ \ t = T, \ldots, 1,
\end{split}
\end{equation}
where $\oplus$ is concatenation. Here, $\latentcolor{\boldz}$ has dimension $d = d_I \times (T + 1)$, where $d_I$ is the image dimension.

\noindent\textbf{Deterministic DPMs: \ } Deterministic DPMs \cite{song2021denoising,0011SKKEP21,Salimans2022ProgressiveDF,Liu2022PseudoNM,Lu2022DPMSolverAF,Karras2022ElucidatingTD,Zhang2022FastSO} generate images with the ODE formulation. Given the mean estimator $\boldmu_{T}$, deterministic DPMs generate \imagecolor{$\boldx := \boldx_0$} via
\begin{align}
\label{eq:ddim}
    \latentcolor{\boldz := \boldx_{T}} \sim \mathcal{N}(\bm{0}, \bm{I}), \quad \boldx_{t-1} = \boldmu_{T}(\boldx_{t}, t), \quad t = T, \ldots, 1.
\end{align}
Since backpropagation through Eq.~(\ref{eq:ddim}) is costly, we use fewer discretization steps $T_{\textit{g}}$ when computing gradients. Given the mean estimator $\boldmu_{T_{\textit{g}}}$ with number of steps $T_{\textit{g}}$, the image \imagecolor{$\boldx := \boldx_0$} is generated as
\begin{align}
\label{eq:ddim-gradient}
    &\text{(with gradients)\quad} \latentcolor{\boldz := \boldx_{T_{\textit{g}}}} \sim \mathcal{N}(\bm{0}, \bm{I}), \quad \boldx_{t-1} = \boldmu_{T_{\textit{g}}}(\boldx_{t}, t), \quad t = T_{\textit{g}}, \ldots, 1.
\end{align}
\noindent\textbf{Latent diffusion models (LDMs):} An LDM \cite{Rombach2021ldm} first uses a diffusion model $\gen_{\text{latent}}$ to compute a ``latent code'' $\boldz_0 = \gen_{\text{latent}}(\latentcolor{\boldz})$,\footnote{\label{footnote:latent-code}Quotation marks stand for ``latent code'' in the cited papers, different from our latent code $\latentcolor{\boldz}$ in Section~\ref{subsec:unify}. } which is then decoded as $\imagecolor{\boldx} = D(\boldz_0)$. Note that $\gen_{\text{latent}}$ is an abstraction of the diffusion models that are already unified above.

\noindent\textbf{Diffusion autoencoder (DiffAE):} DiffAE \cite{Preechakul2021DiffusionAT} first uses a deterministic DDIM to generate a ``latent code'' $\boldz_0$,\footref{footnote:latent-code} which is used as condition for an image-space deterministic DDIM:
\begin{equation}
\label{eq:diffae}
\begin{split}
    & \latentcolor{\boldz := \big(\boldz_T \oplus \boldx_T\big)} \sim \mathcal{N}(\bm{0}, \bm{I}), \quad \boldz_0 = \text{DDIM}_{Z} (\latentcolor{\boldz_T}), \quad \imagecolor{\boldx := \boldx_0} = \text{DDIM}_{X} (\latentcolor{\boldx_T}, \boldz_0).
\end{split}
\end{equation}
\noindent\textbf{DDGAN: } DDGAN \cite{xiao2022tackling} models each reverse time step $t$ as a GAN conditioned on the output of the previous step. We define the latent code $\latentcolor{\boldz}$ and generation process $\gen$ of DDGAN as 
\begin{equation}
\label{eq:ddgan}
\begin{split}
    &\latentcolor{\boldz := \big(\boldx_T \oplus \boldz_T \oplus \boldeps_T \oplus \cdots \oplus \boldz_{2} \oplus \boldeps_{2} \oplus \boldz_{1}\big)} \sim \mathcal{N}(\bm{0}, \bm{I}), \\
    &\boldx_{t-1} = \boldmu_{T}(\boldx_{t}, \latentcolor{\boldz_t}, t) + \boldsigma_t \odot \latentcolor{\boldeps_{t}}, \quad t = T, \ldots, 2, \quad \imagecolor{\boldx := \boldx_{0}} = \boldmu_{T}(\boldx_{1}, \latentcolor{\boldz_1}, 1).
\end{split}
\end{equation}

\begin{algorithm}[!ht]
\DontPrintSemicolon
\textbf{Input:} source image $\imagecolor{\boldx := \boldx_0}$; source text $\bm{t}$; target text $\hat{\bm{t}}$; encoding step $T_{\text{es}} \leq T$\;
1. Sample noisy image $\hat{\boldx}_{T_{\text{es}}} = \boldx_{T_{\text{es}}} \sim q(\boldx_{T_{\text{es}}}|\imagecolor{\boldx_0})$\;
\For{$t = T_{\text{es}}, \ldots, 1$}{
    2. $\boldx_{t-1} \sim q(\boldx_{t-1}|\boldx_t, \imagecolor{\boldx_0})$\;
    3. $\latentcolor{\boldeps_t} = \big(\boldx_{t-1} - \boldmu_{T}(\boldx_{t}, t | \bm{t})\big) / \boldsigma_t$\;
    4. $\hat{\boldx}_{t-1} = \boldmu_{T}(\hat{\boldx}_{t}, t | \hat{\bm{t}}) + \boldsigma_t \odot \latentcolor{\boldeps_t}$\;
}
\textbf{Output: } $\imagecolor{\hat{\boldx} := \hat{\boldx}_0}$\;
	\caption{\cyclediff for zero-shot image-to-image translation }
	\label{alg:cycle-diffusion-early-stop-text}
\end{algorithm}

\subsection{DPM-Encoder: A Reconstructable Encoder for Diffusion Models}
\label{subsec:encoding}

In this section, we investigate the \textit{encoding} problem, i.e., $\latentcolor{\boldz} \sim \mathrm{Enc}(\latentcolor{\boldz}|\imagecolor{\boldx}, \gen)$. The encoding problem has been studied for many generative models, and our contribution is \textbf{DPM-Encoder}, an encoder for stochastic DPMs. DPM-Encoder is defined as follows. For each image \imagecolor{$\boldx := \boldx_0$}, stochastic DPMs define a posterior distribution $q(\boldx_{1:T}|\imagecolor{\boldx_0})$ \cite{HoJA20,song2021denoising}. Based on $q(\boldx_{1:T}|\imagecolor{\boldx_0})$ and Eq.~(\ref{eq:diffusion-as-implicit-latent}), we can directly derive $\latentcolor{\boldz} \sim \mathrm{DPMEnc}(\latentcolor{\boldz}|\imagecolor{\boldx}, \gen)$ as (see details in Appendices~\ref{app:math-details} and \ref{app:dpm-encoder})
\begin{equation}
\label{eq:dpm-encoder}
\begin{split}
    &\boldx_1, \ldots, \boldx_{T-1}, \latentcolor{\boldx_T} \sim q(\boldx_{1:T}|\imagecolor{\boldx_0}), \quad \latentcolor{\boldeps_t} = \big(\boldx_{t-1} - \boldmu_{T}(\boldx_{t}, t)\big) / \boldsigma_t, \quad t = T, \ldots, 1, \\
    &\quad\quad\quad\quad\quad\quad\quad\quad\quad\quad  \latentcolor{\boldz := \big(\boldx_T \oplus \boldeps_T \oplus \cdots \oplus \boldeps_{2} \oplus \boldeps_{1}\big)}.
\end{split}
\end{equation}
A property of DPM-Encoder is perfect reconstruction, meaning that we have $\imagecolor{\boldx} = \gen(\latentcolor{\boldz})$ for every $\latentcolor{\boldz} \sim \mathrm{Enc}(\latentcolor{\boldz}|\imagecolor{\boldx}, \gen)$. A proof by induction is provided in Appendix~\ref{app:dpm-encoder}.

\subsection{\cyclediff: Image-to-Image Translation with DPM-Encoder}
\label{subsec:cyclediff-method}

Given two stochastic DPMs $\gen_1$ and $\gen_2$ that model two distributions $\mathcal{D}_1$ and $\mathcal{D}_2$, several researchers and practitioners have found that sampling with the same ``random seed'' leads to similar images \cite{Nichol2021GLIDETP}. To formalize ``similar images'', we provide an upper bound of image distances based on assumptions about the trained DPMs, shown at the end of this subsection. Based on this finding, we propose a simple unpaired image-to-image translation method, \cyclediff. Given a source image $\imagecolor{\boldx} \in \mathcal{D}_1$, we use DPM-Encoder to encode it as $\latentcolor{\boldz}$ and then decode it as $\imagecolor{\hat{\boldx}} = \gen_2(\latentcolor{\boldz})$:
\begin{equation}
\label{eq:unit-dpm-encoder}
    \latentcolor{\boldz} \sim \mathrm{DPMEnc}(\latentcolor{\boldz}|\imagecolor{\boldx}, \gen_1), \quad \imagecolor{\hat{\boldx}} = \gen_2(\latentcolor{\boldz}).
\end{equation}

We can also apply \cyclediff to text-to-image diffusion models by defining $\mathcal{D}_1$ and $\mathcal{D}_2$ as image distributions conditioned on two texts. Let $\gen_{\bm{t}}$ be a text-to-image diffusion model conditioned on text $\bm{t}$. Given a source image $\imagecolor{\boldx}$, the user writes two texts: a source text $\bm{t}$ describing the source image $\imagecolor{\boldx}$ and a target text $\hat{\bm{t}}$ describing the target image $\imagecolor{\hat{\boldx}}$ to be generated. 
We can then perform zero-shot image-to-image editing via (zero-shot means that the model has never been trained on image editing) 
\begin{equation}
\label{eq:dalle-dpm-encoder}
    \latentcolor{\boldz} \sim \mathrm{DPMEnc}(\latentcolor{\boldz}|\imagecolor{\boldx}, \gen_{\bm{t}}), \quad \imagecolor{\hat{\boldx}} = \gen_{\hat{\bm{t}}}(\latentcolor{\boldz}).
\end{equation}
Inspired by the realism-faithfulness tradeoff in SDEdit \cite{meng2022sdedit}, we can truncate $\latentcolor{\boldz}$ towards a specified encoding step $T_{\text{es}} \leq T$. The algorithm of \cyclediff is shown in Algorithm~\ref{alg:cycle-diffusion-early-stop-text}. 

\noindent\textbf{An analysis for image similarity with fixed $\latentcolor{\boldz}$.}  \ We analyze the image similarity using text-to-image diffusion models. Suppose the text-to-image model has the following two properties: 

\begin{enumerate}
    \item Conditioned on the same text, similar noisy images lead to similar enough mean predictions. Formally, $\boldmu_{T}(\boldx_{t}, t|\bm{t})$ is $K_t$-Lipschitz, i.e., $\|\boldmu_{T}(\boldx_{t}, t|\bm{t}) - \boldmu_{T}(\hat{\boldx}_{t}, t|\bm{t})\|\leq K_t\|\boldx_{t} - \hat{\boldx}_{t}\|$.
    \item Given the same image, the two texts lead to similar predictions. Formally, $\|\boldmu_{T}(\hat{\boldx}_{t}, t|\bm{t}) - \boldmu_{T}(\hat{\boldx}_{t}, t|\hat{\bm{t}})\| \leq S_t$. Intuitively, a smaller difference between $\bm{t}$ and $\hat{\bm{t}}$ gives us a smaller $S_t$. 
\end{enumerate}
Let $B_{t}$ be the upper bound of $\|\boldx_{t} - \hat{\boldx}_{t}\|_2$ at time step $t$ when the same latent code $\latentcolor{\boldz}$ is used for sampling (i.e., $\imagecolor{\boldx_0} = G_{\bm{t}}(\latentcolor{\boldz})$ and $\imagecolor{\hat{\boldx}_0} = G_{\hat{\bm{t}}}(\latentcolor{\boldz})$). We have $B_{T} = 0$ because $\|\latentcolor{\boldx_T} - \latentcolor{\hat{\boldx}_T}\|_2 = 0$, and $B_{0}$ is the upper bound for the generated images $\|\imagecolor{\boldx} - \imagecolor{\hat{\boldx}}\|_2$. The upper bound $B_{t}$ can be propagated through time, from $T$ to $0$. Specifically, by combining the above two properties, we have 
\begin{equation}
    B_{t-1} \leq (K_t + 1) B_t + S_t. 
\end{equation}

\subsection{Unified Plug-and-Play Guidance for Generative Models}
\label{subsec:control}
Prior works showed that guidance for generative models can be achieved in the latent space \cite{Nguyen2017PlugP,nie2021controllable,promptgen}. Specifically, given a condition $\controlcolor{\mathcal{C}}$, one can define the guided image distribution as an energy-based model (EBM): $p(\imagecolor{\boldx}|\controlcolor{\mathcal{C}}) \propto p_{\imagecolor{\boldx}}(\imagecolor{\boldx})e^{-\lambda_{\controlcolor{\mathcal{C}}}E(\imagecolor{\boldx}|{\controlcolor{\mathcal{C}}})}$.
Sampling for $\imagecolor{\boldx} \sim p(\imagecolor{\boldx}|\controlcolor{\mathcal{C}})$ is equivalent to $\latentcolor{\boldz} \sim p_{\latentcolor{\boldz}}(\latentcolor{\boldz}|\controlcolor{\mathcal{C}}), \imagecolor{\boldx} = \gen(\latentcolor{\boldz})$, where $p(\latentcolor{\boldz}|\controlcolor{\mathcal{C}}) \propto p_{\latentcolor{\boldz}}(\latentcolor{\boldz})e^{-\lambda_{\controlcolor{\mathcal{C}}}E(\gen(\latentcolor{\boldz})|{\controlcolor{\mathcal{C}}})}$.
Examples of the energy function $E(\imagecolor{\boldx}|{\controlcolor{\mathcal{C}}})$ are provided in Section~\ref{subsec:control-experiment}. To sample $\latentcolor{\boldz} \sim p(\latentcolor{\boldz}|\controlcolor{\mathcal{C}})$, one can use any model-agnostic samplers. For example, Langevin dynamics \cite{Welling2011BayesianLV} starts from $\latentcolor{\boldz^{\langle0\rangle}} \sim \mathcal{N}(\bm{0}, \bm{I})$ and samples $\latentcolor{\boldz := \boldz^{\langle n\rangle}}$ iteratively through
\begin{equation}
\label{eq:langevin}
\begin{split}
    \latentcolor{\boldz^{\langle k+1\rangle}}
    = \latentcolor{\boldz^{\langle k\rangle}} + \frac{\sigma}{2} \nabla_{\latentcolor{\boldz}} \Big( \log p_{\latentcolor{\boldz}}(\latentcolor{\boldz^{\langle k\rangle}}) -E\big(\gen(\latentcolor{\boldz^{\langle k\rangle}}) | {\controlcolor{\mathcal{C}}}\big) \Big) + \sqrt{\sigma}\bm{\omega}^{\langle k\rangle}, \quad \bm{\omega}^{\langle k\rangle} \sim \mathcal{N}(\bm{0}, \bm{I}).
\end{split}
\end{equation}

\begin{table}[!th]
    \caption{Quantitative comparison for unpaired image-to-image translation methods. Methods in the second block use the same pre-trained diffusion model in the target domain. Results of CUT, ILVR, SDEdit, and EGSDE are from \citet{Zhao2022EGSDEUI}. Best results using diffusion models are in \textbf{bold}. \cyclediff has the best FID and KID among all methods and the best SSIM among methods with diffusion models. Note that it has been shown that SSIM is much better correlated with human visual perception than squared distance-based metrics such as $L_2$ and \psnr~\cite{Wang2004ImageQA}. }\smallskip
    \label{tab:translate-results}
    \centering
    \begin{adjustbox}{width=\linewidth}
    \begin{tabular}{@{}l@{}cccccccc@{}}
        \toprule
        & \multicolumn{4}{c}{Cat $\rightarrow$ Dog} & \multicolumn{4}{c}{Wild $\rightarrow$ Dog} \\
        \cmidrule(r){2-5} \cmidrule(l){6-9}
        & FID\down & KID$\times10^{3}$\down & \psnr\up & SSIM\up & FID\down & KID$\times10^{3}$\down & \psnr\up & SSIM\up \\
        \midrule
        CUT (GAN SOTA; \citealp{Park2020ContrastiveLF}) \ \ & 76.21     & --        & 17.48     & 0.601     & 92.94     & --        & 17.20     & 0.592 \\
        \midrule
        ILVR \cite{Choi2021ILVRCM}                          & 74.37     & --        & 17.77     & 0.363     & 75.33     & --        & 16.85     & 0.287 \\
        SDEdit \cite{meng2022sdedit}                        & 74.17     & --        & 19.19     & 0.423     & 68.51     & --        & 17.98     & 0.343 \\
        EGSDE \cite{Zhao2022EGSDEUI}                        & 65.82     & --        & \bf 19.31 & 0.415     & 59.75     & --        & \bf 18.14 & 0.343 \\
        \cyclediff w/ DDIM ($\eta = 0.1$)                   & \bf 58.87 & 20.3      & 18.50     & \bf 0.557 & \bf 56.45 & 19.5      & 17.82     & \bf 0.479 \\
        \bottomrule
    \end{tabular}
    \end{adjustbox}
\end{table}

\section{Experiments}
\label{sec:experiments}

This section provides experimental validation of the proposed work. Section~\ref{subsec:align} shows how \cyclediff achieves competitive results on unpaired image-to-image translation benchmarks. Section~\ref{subsec:dalle} provides a protocol for what we call zero-shot image-to-image translation; \cyclediff outperforms several image-to-image translation baselines that we re-purposed for this new task.
Section~\ref{subsec:control-experiment} shows how diffusion models and GANs can be guided in a unified, plug-and-play formulation. 

\subsection{\cyclediff for Unpaired Image-to-Image Translation}
\label{subsec:align}

Given two unaligned image domains, unpaired image-to-image translation aims at mapping images in one domain to the other. We follow setups from previous works whenever possible, as detailed below. 
Following previous work \cite{Park2020ContrastiveLF,Zhao2022EGSDEUI}, we conducted experiments on the test set of AFHQ \cite{choi2020starganv2} with resolution $256 \times 256$ for Cat $\rightarrow$ Dog and Wild $\rightarrow$ Dog. For each source image, each method should generate a target image with minimal changes. Since \cyclediff sometimes generates noisy outputs, we used $T_{\text{sdedit}}$ steps of SDEdit for denoising. When $T = 1000$, we set $T_{\text{sdedit}} = 100$ for Cat $\rightarrow$ Dog and $T_{\text{sdedit}} = 125$ for Wild $\rightarrow$ Dog. 

\noindent\textbf{Metrics: } To evaluate realism, we reported Frechet Inception Distance (FID; \citealp{Heusel2017GANsTB}) and Kernel Inception Distance (KID; \citealp{bikowski2018demystifying}) between the generated and target images. To evaluate faithfulness, we reported Peak Signal-to-Noise Ratio (\psnr)~and Structural Similarity Index Measure (SSIM; \citealp{Wang2004ImageQA}) between each generated image and its source image. 

\noindent\textbf{Baselines: } We compared \cyclediff with previous state-of-the-art unpaired image-to-image translation methods: CUT \cite{Park2020ContrastiveLF}, ILVR \cite{Choi2021ILVRCM}, SDEdit \cite{meng2022sdedit}, and EGSDE \cite{Zhao2022EGSDEUI}. CUT is based on GAN, and the others use diffusion models.  

\noindent\textbf{Pre-trained diffusion models: } ILVR, SDEdit, and EGSDE only need the diffusion model trained on the target domain, and we followed them to use the pre-trained model from \citet{Choi2021ILVRCM} for Dog. \cyclediff needs diffusion models on both domains, so we trained them on Cat and Wild.

Seen in Table~\ref{tab:translate-results} are the results. \cyclediff has the best realism (i.e., FID and KID). There is a mismatch between the faithfulness metrics (i.e., \psnr~and SSIM), and note that SSIM is much better correlated with human perception than \psnr~\cite{Wang2004ImageQA}.
Among all diffusion model-based methods, \cyclediff achieves the highest SSIM. Figure~\ref{fig:cyclediff-comparison} displays some image samples from \cyclediff, showing that our method can change the domain while preserving local details such as the background, lighting, pose, and overall color the animal.

\begin{figure}[!ht]
\centering
    \includegraphics[width=\linewidth]{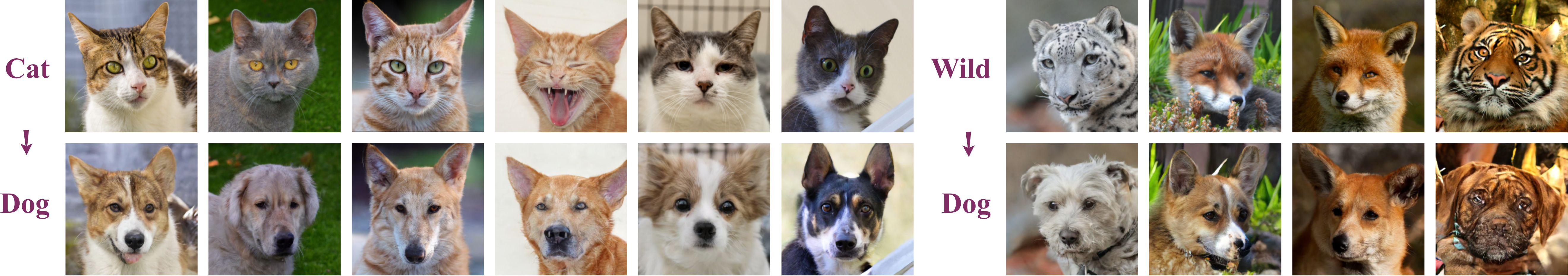}
\caption{\label{fig:cyclediff-comparison} Unpaired image-to-image translation (Cat $\rightarrow$ Dog, Wild $\rightarrow$ Dog) with \cyclediff. } 
\end{figure}

\subsection{Text-to-Image Diffusion Models Can Be Zero-Shot Image-to-Image Editors}
\label{subsec:dalle}

This section provides experiments for zero-shot image editing. 
We curated a set of 150 tuples $(\imagecolor{\boldx}, \bm{t}, \hat{\bm{t}})$ for this task, where $\imagecolor{\boldx}$ is the source image, $\bm{t}$ is the source text (e.g., ``an aerial view of autumn scene.'' in Figure~\ref{fig:dalle-experiment} second row on the right), and $\hat{\bm{t}}$ is the target text (e.g., ``an aerial view of winter scene.''). The generated image is denoted as $\imagecolor{\hat{\boldx}}$. We also demonstrate that \cyclediff can be combined with the Cross Attention Control \cite{Hertz2022PrompttoPromptIE} to further preserve the image structure. 

\noindent\textbf{Metrics: } To evaluate the faithfulness of the generated image to the source image, we reported \psnr~and SSIM. To evaluate the authenticity of the generated image to the target text, we reported the CLIP score $\mathcal{S}_{\text{CLIP}}(\imagecolor{\hat{\boldx}} | \hat{\bm{t}}) =  \cos\big\langle\text{CLIP}_{\text{img}}(\imagecolor{\hat{\boldx}}), \text{CLIP}_{\text{text}}(\hat{\bm{t}})\big\rangle$, where the CLIP embeddings are normalized. We note a trade-off between \psnr/SSIM and $\mathcal{S}_{\text{CLIP}}$: by copying the source image we get high \psnr/SSIM but low $\mathcal{S}_{\text{CLIP}}$, and by ignoring the source image (e.g., by directly generating images conditioned on the target text) we get high $\mathcal{S}_{\text{CLIP}}$ but low \psnr/SSIM. 
To address this trade-off, we also reported the directional CLIP score \cite{Patashnik2021StyleCLIPTM} (the CLIP embeddings are normalized):
\begin{equation}
\label{eq:directional-clip-score}
    \mathcal{S}_{\text{D-CLIP}}(\imagecolor{\hat{\boldx}} | \imagecolor{\boldx}, \bm{t}, \hat{\bm{t}}) =  \cos\Big\langle\text{CLIP}_{\text{img}}(\imagecolor{\hat{\boldx}}) - \text{CLIP}_{\text{img}}(\imagecolor{\boldx}), \text{CLIP}_{\text{text}}(\hat{\bm{t}}) - \text{CLIP}_{\text{text}}(\bm{t}) \Big\rangle.
\end{equation}

\noindent\textbf{Baselines: } The baselines include SDEdit \cite{meng2022sdedit} and DDIB \cite{Su2022DualDI}. We used the same hyperparameters for the baselines and \cyclediff whenever possible (e.g., the number of diffusion steps, the strength of classifier-free guidance; see Appendix~\ref{app:experimental-details}). 

\noindent\textbf{Pre-trained text-to-image diffusion models:} We used the following text-to-image diffusion models models: (1) \texttt{LDM-400M,} a 1.45B-parameter model trained on LAION-400M \cite{Schuhmann2021LAION400MOD}, (2) \texttt{SD-v1-4,} a 0.98B-parameter Stable Diffusion trained on LAION-5B \cite{laion5b}.

\noindent\textbf{Results:} Table~\ref{tab:dalle-results} shows the results for zero-shot image-to-image translation. \cyclediff excels at being faithful to the source image (i.e., \psnr~and SSIM); by contrast, SDEdit and DDIB have comparable authenticity to the target text (i.e., $\mathcal{S}_{\text{CLIP}}$), but their outputs are much less faithful. 
For all methods, we find that the pre-trained weights \texttt{SD-v1-1} and \texttt{SD-v1-4} have better faithfulness than \texttt{LDM-400M}. Figure~\ref{fig:dalle-experiment} provides samples from \cyclediff, demonstrating that \cyclediff achieves meaningful edits that span (1) replacing objects, (2) adding objects, (3) changing styles, and (4) modifying attributes. See Figure~\ref{fig:dalle-experiment-baseline} (Appendix~\ref{app:dalle-details}) for qualitative comparisons with the baselines. 

\noindent\textbf{\cyclediff + Cross Attention Control: } Besides fixing the random seed, \citet{Hertz2022PrompttoPromptIE} shows that fixing the cross attention map (i.e., Cross Attention Control, or CAC) further improves the similarity between synthesized images. CAC is applicable to \cyclediff: in Algorithm~\ref{alg:cycle-diffusion-early-stop-text}, we can apply the attention map of $\boldmu_{T}(\boldx_{t}, t | \bm{t})$ to $\boldmu_{T}(\hat{\boldx_{t}}, t | \hat{\bm{t}})$. However, we cannot apply it to all samples because CAC puts requirements on the difference between $\bm{t}$ and $\hat{\bm{t}}$. Figure~\ref{fig:cyclediff-cac} shows that CAC helps \cyclediff when the intended \textit{structural} change is small. For instance, when the intended change is color but not shape (left), CAC helps \cyclediff preserve the background; when the intended change is horse $\rightarrow$ elephant, CAC makes the generated elephant to look more like a horse in shape.

\begin{table}[!th]
    \caption{\label{tab:dalle-results} Zero-shot image editing. We did not use fixed hyperparameters, and neither did we plot the trade-off curve. The reason is that every input can have its best hyperparameters and even random seed. Instead, \textbf{for each input}, we ran 15 random trials for each hyperparameter and report the one with the highest $\mathcal{S}_{\text{D-CLIP}}$. For a fair comparison, different methods share the same set of combinations of hyperparameters if possible, detailed in Appendix~\ref{app:experimental-details}. }\smallskip
    \centering
    \begin{adjustbox}{width=0.9\linewidth}
    \begin{tabular}{@{}llcccc@{}}
        \toprule
        & Method & $\mathcal{S}_{\text{CLIP}}$\up & $\mathcal{S}_{\text{D-CLIP}}$\up & \psnr\up & SSIM\up \\
        \midrule
        \multirow{3}*{\texttt{LDM-400M}} & SDEdit \cite{meng2022sdedit} & 0.332     & 0.264     & 13.68     & 0.390 \\
        & DDIB \cite{Su2022DualDI}                                      & 0.324     & 0.195     & 15.82     & 0.544 \\
        & \cyclediff w/ DDIM ($\eta = 0.1$; ours)                       & \bf 0.333 & \bf 0.275 & \bf 18.72 & \bf 0.625 \\
        \midrule
        \multirow{3}*{\texttt{SD-v1-4}} & SDEdit \cite{meng2022sdedit}  & \bf 0.344 & 0.258     & 15.93     & 0.512 \\
        & DDIB \cite{Su2022DualDI}                                      & 0.331     & 0.209     & 18.10     & 0.653 \\
        & \cyclediff w/ DDIM ($\eta = 0.1$; ours)                       & 0.334     & \bf 0.272 & \bf 21.92 & \bf 0.731 \\
        \bottomrule
    \end{tabular}
    \end{adjustbox}
\end{table}

\begin{figure}[!th]
\centering
    \includegraphics[width=\linewidth]{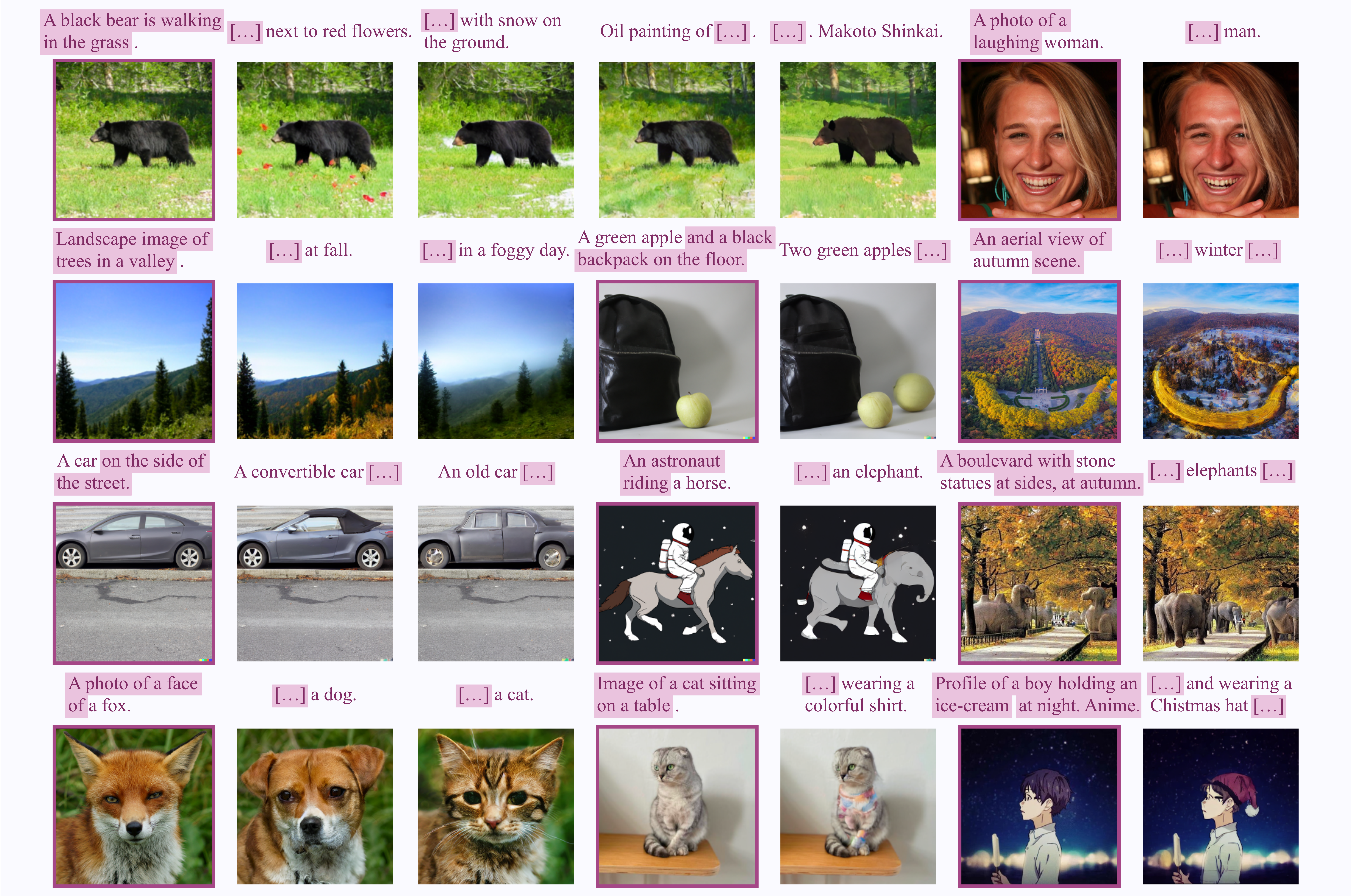}
\caption{\label{fig:dalle-experiment} \cyclediff for zero-shot image editing. Source images $\imagecolor{\boldx}$ are displayed with a purple margin; the other images are the generated $\imagecolor{\hat{\boldx}}$. Within each pair of source and target texts, overlapping text spans are marked in purple in the source text and abbreviated as $[\ldots]$ in the target text. 
}
\end{figure}

\begin{figure}[!ht]
\centering
    \includegraphics[width=\linewidth]{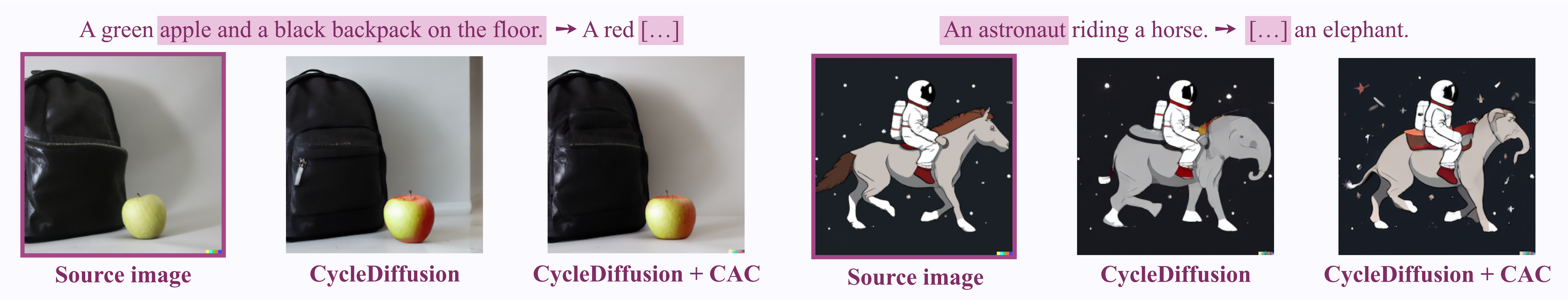}
\caption{\label{fig:cyclediff-cac} Cross Attention Control (CAC; \citealp{Hertz2022PrompttoPromptIE}) helps \cyclediff when the intended \textit{structural} change is small. For instance, when the intended change is color but not shape (left), CAC helps \cyclediff preserve the background; when the intended change is horse $\rightarrow$ elephant, CAC makes the generated elephant look more like a horse in shape. } 
\end{figure}

\begin{figure}[!th]
\centering
    \subfigure[\label{subfig:baby} a baby\baby]{
        \includegraphics[width=0.482\linewidth]{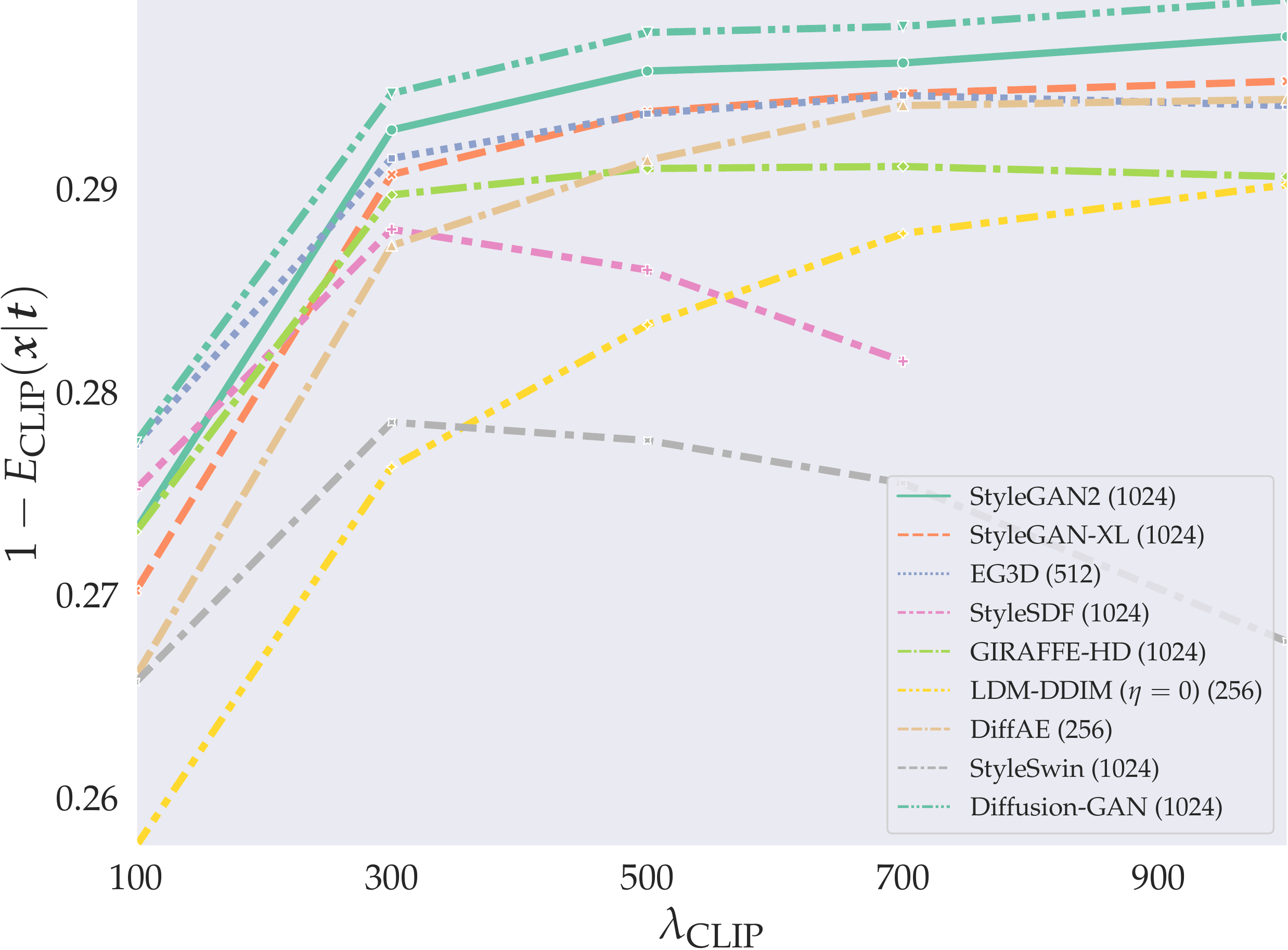}
    }
    \subfigure[\label{subfig:old-person} an old person\oldperson]{
        \includegraphics[width=0.482\linewidth]{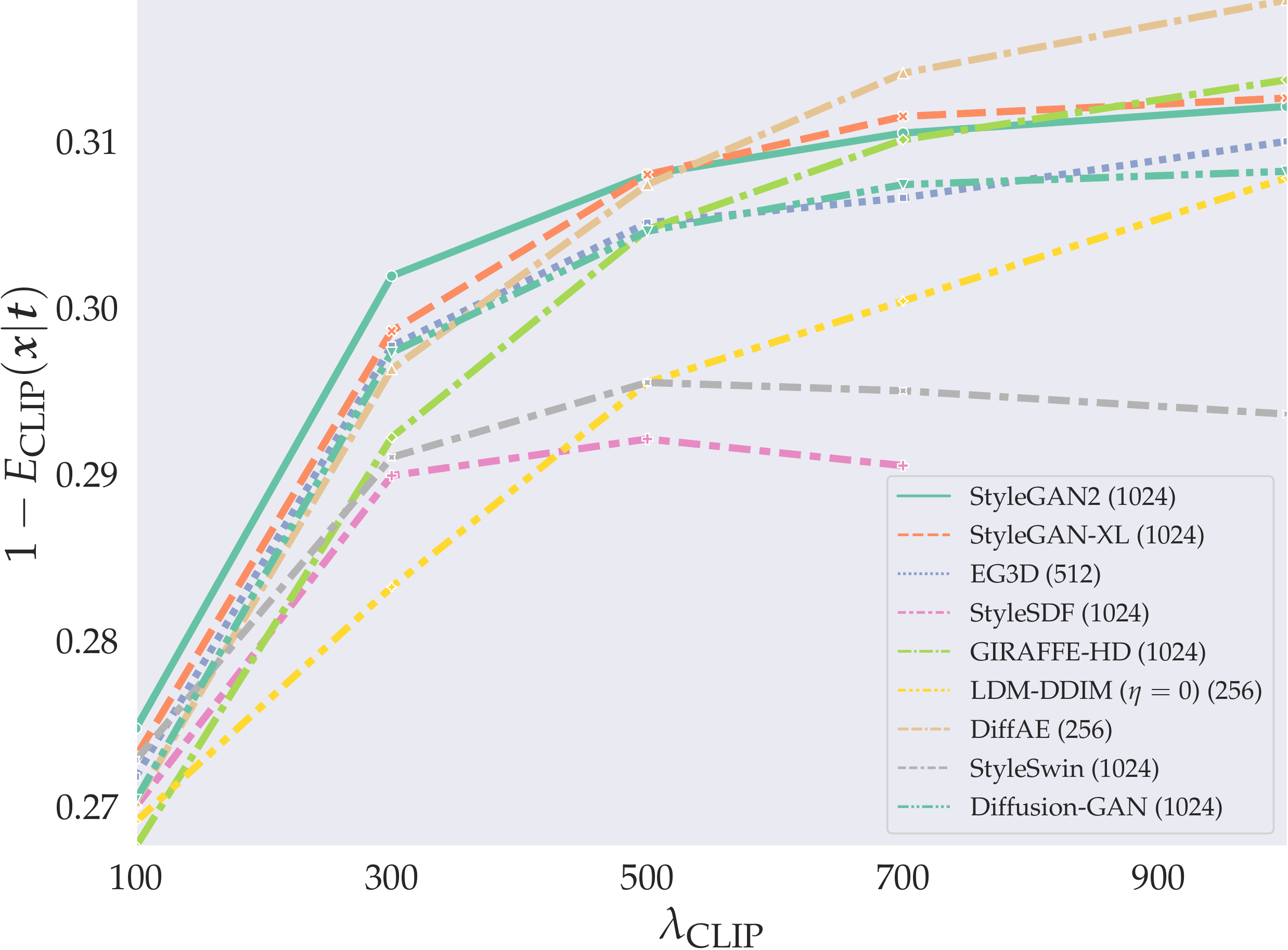}
    }
    \subfigure[\label{subfig:eyeglasses} a person with eyeglasses\eyeglasses]{
        \includegraphics[width=0.482\linewidth]{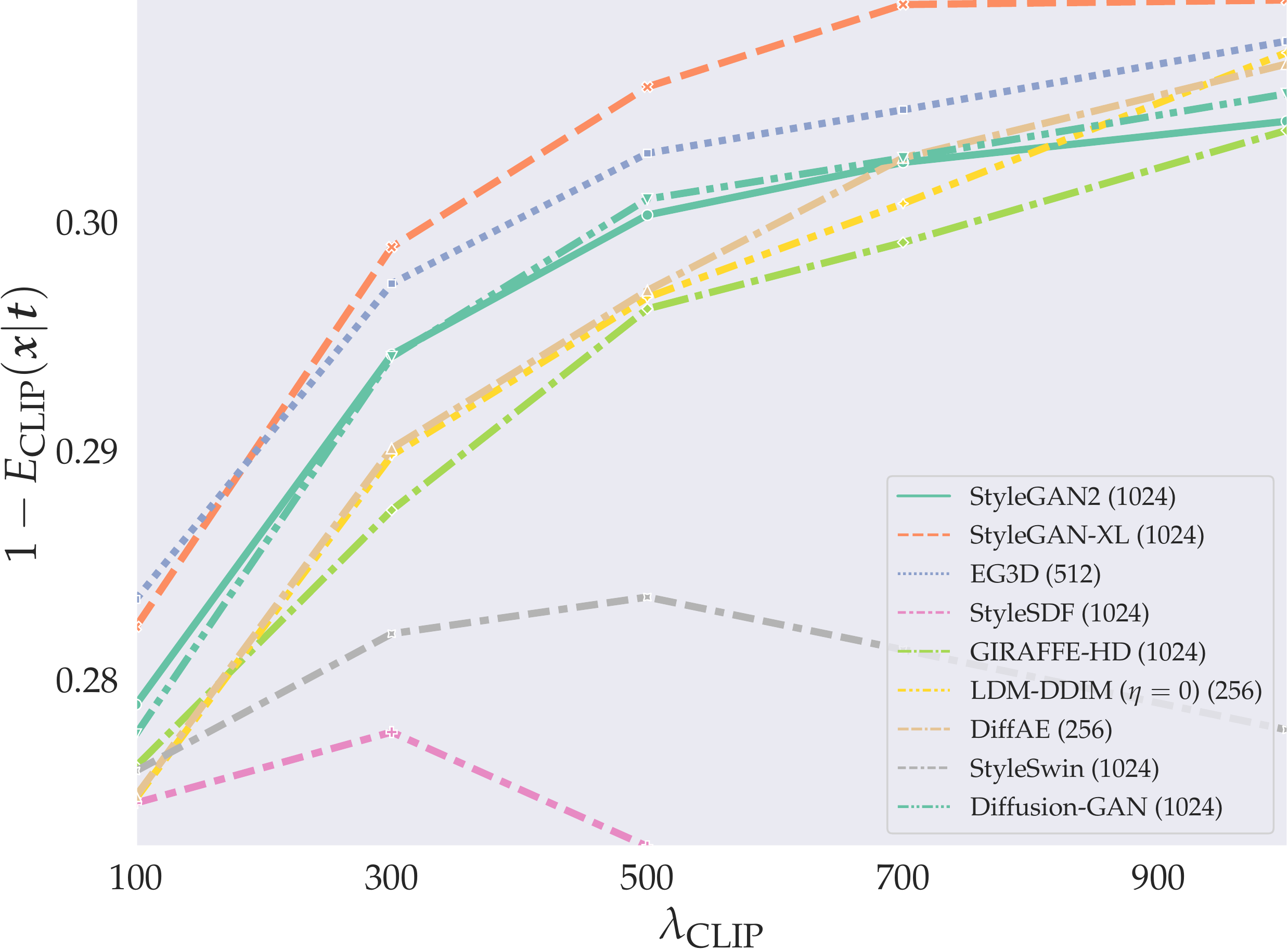}
    }
    \subfigure[\label{subfig:eyeglasses-yellow-hat} a person with eyeglasses and a yellow hat\eyeglasses\yellowhat]{
        \includegraphics[width=0.482\linewidth]{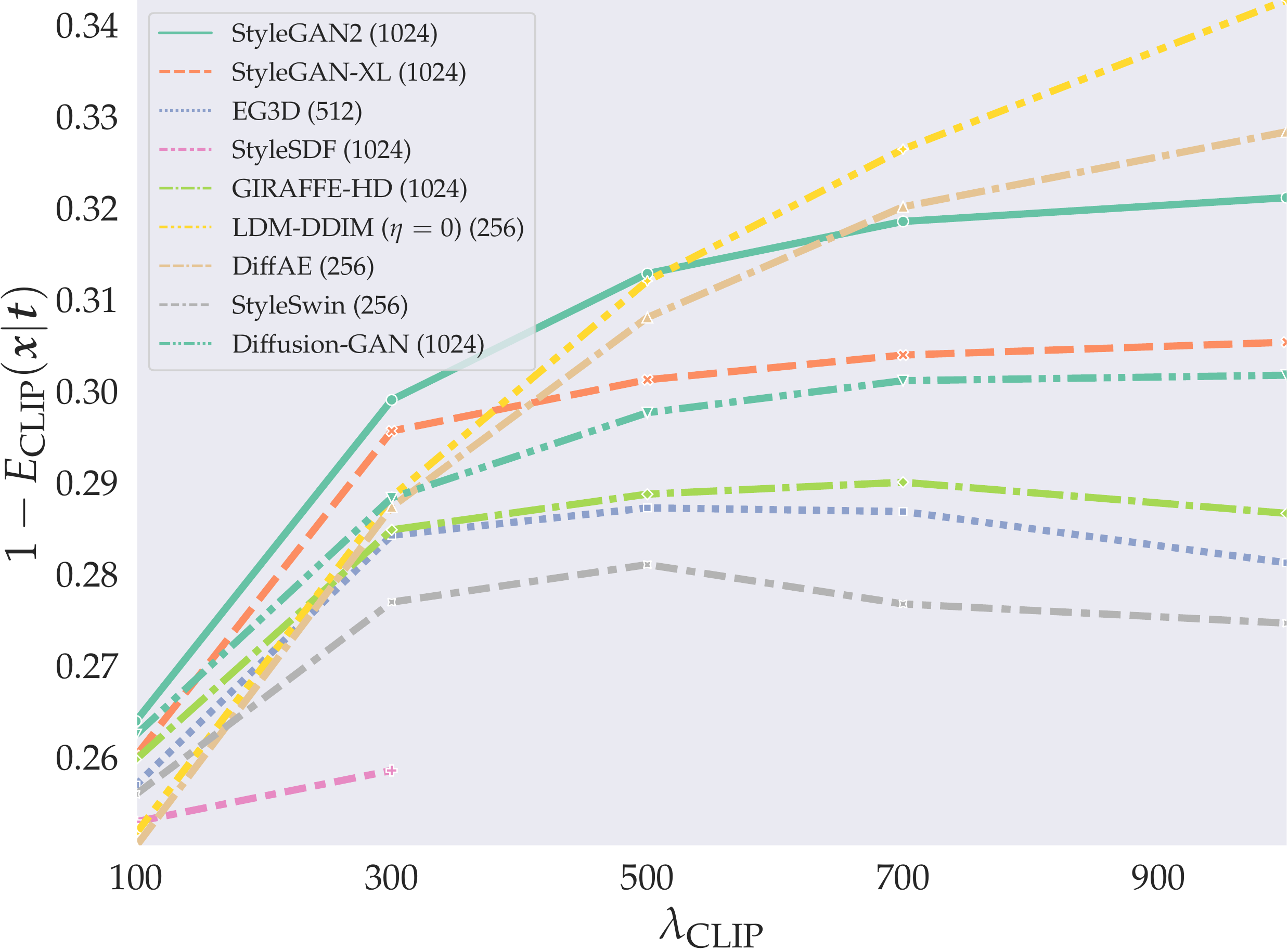}
    }
\caption{\label{fig:clip-curves} Unified plug-and-play guidance for diffusion models and GANs with text and CLIP. The text description used in each plot is \underline{a }p\underline{hoto of }$\big[$\underline{\quad\quad}$\big]$. Image samples and more analyses are in Figure~\ref{fig:clip-samples} and Appendix~\ref{app:more-coverage}. When the guidance becomes complex, diffusion models surpass GANs. } 
\end{figure}

\subsection{Unified Plug-and-Play Guidance for Diffusion Models and GANs}
\label{subsec:control-experiment}

Previous methods for conditional sampling from (\textit{aka} guiding) diffusion models require training the guidance model on noisy images \cite{dhariwal2021diffusion,Liu2021MoreCF}, which deviates from the idea of plug-and-play guidance by leveraging the simple latent prior of generative models \cite{Nguyen2017PlugP}. In contrast, our definition of the Gaussian latent space of different diffusion models allows for unified plug-and-play guidance of diffusion models and GANs. It facilitates principled comparisons over sub-populations and individuals when models are trained on the same dataset. 

We used the text $\controlcolor{\bm{t}}$ to specify sub-population. For instance, \underline{a }p\underline{hoto of bab}y represents the baby sub-population in the domain of human faces. We instantiate the energy in Section~\ref{subsec:control} as $E_{\text{CLIP}}(\imagecolor{\boldx} | \controlcolor{\bm{t}}) = \frac{1}{L} \sum_{l=1}^{L}\Big(1 - \cos\big\langle\text{CLIP}_{\text{img}}\big(\text{DiffAug}_{l}(\imagecolor{\boldx})\big), \text{CLIP}_{\text{text}}(\controlcolor{\bm{t}})\big\rangle\Big)$, where DiffAug$_{l}$ stands for differentiable augmentation \cite{ZhaoLLZ020} that mitigates the adversarial effect, and we sample from the energy-based distribution using Langevin dynamics in Eq.~(\ref{eq:langevin}) with $n=200$, $\sigma = 0.05$. We enumerated the guidance strength (i.e., the coefficient $\lambda_{\controlcolor{\mathcal{C}}}$ in Section~\ref{subsec:control}) $\lambda_{\text{CLIP}} \in \{100, 300, 500, 700, 1000\}$. For evaluation, we reported $(1 - E_{\text{CLIP}}(\imagecolor{\boldx} | \controlcolor{\bm{t}}))$ averaged over 256 samples. This metric quantifies whether the sampled images are consistent with the specified text $\controlcolor{\bm{t}}$. 
Figure~\ref{fig:clip-curves} plots models with pre-trained weights on FFHQ \cite{Karras2019ASG} (citations in Table~\ref{tab:models-included}, Appendix~\ref{app:more-coverage}). 
In Figure~\ref{fig:clip-samples}, we visualize samples for SN-DDPM and DDGAN trained on CelebA. 
We find that diffusion models outperform 2D/3D GANs for complex text, and different models represent the same sub-population differently. 

Broad coverage of individuals is an important aspect of the personalized use of generative models. To analyze this coverage, we guide different models to generate images that are close to a reference $\controlcolor{\boldx_r}$ in the identity (ID) space modeled by the IR-SE50 face embedding model \cite{Deng2019ArcFaceAA}, denoted as $R$. Given an ID reference image $\controlcolor{\boldx_r}$, we instantiated the energy defined in Section~\ref{subsec:control} as $E_{\text{ID}}(\imagecolor{\boldx} | \controlcolor{\boldx_r}) = 1 - \cos\big\langle R(\imagecolor{\boldx}), R(\controlcolor{\boldx_r})\big\rangle$ with strength $\lambda_{\text{ID}}=2500$ (i.e., $\lambda_{\controlcolor{\mathcal{C}}}$ in Section~\ref{subsec:control}). For sampling, we used Langevin dynamics detailed in Eq.~(\ref{eq:langevin}) with $n=200$ and $\sigma = 0.05$. 
To measure ID similarity to the reference image $\controlcolor{\boldx_r}$, we reported $\cos\big\langle R(\imagecolor{\boldx}), R(\controlcolor{\boldx_r})\big\rangle$, averaged over 256 samples. 
In Table~\ref{tab:id-results}, we report the performance of StyleGAN2, StyleGAN-XL, GIRAFFE-HD, EG3D, LDM-DDIM, DDGAN, and DiffAE. DDGAN is trained on CelebAHQ, while others are trained on FFHQ. 
We find that diffusion models have much better coverage of individuals than 2D/3D GANs. Among diffusion models, deterministic LDM-DDIM ($\eta=0$) achieves the best identity guidance performance. We provide image samples of identity guidance in Figure~\ref{fig:individual} (Appendix~\ref{app:more-coverage}).

\begin{figure}[!ht]
\centering
    \subfigure[\label{subfig:baby-stylegan2} StyleGAN2\baby]{
        \includegraphics[width=0.31\linewidth]{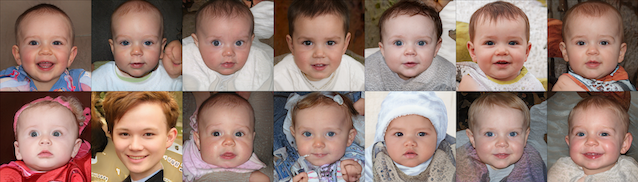}  
    }
    \subfigure[\label{subfig:baby-styleganxl} StyleGAN-XL\baby]{
        \includegraphics[width=0.31\linewidth]{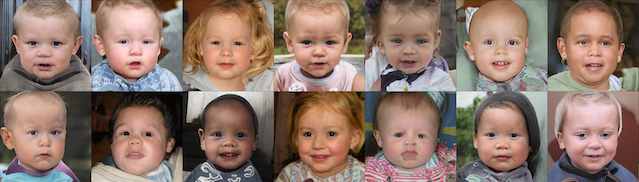}  
    }
    \subfigure[\label{subfig:baby-giraffehd} GIRAFFE HD\baby]{
        \includegraphics[width=0.31\linewidth]{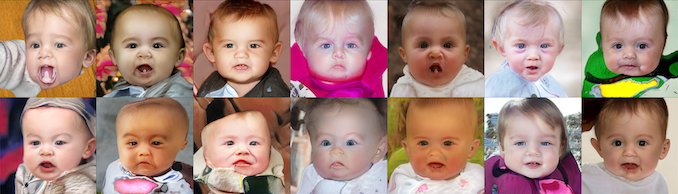}  
    }
    \subfigure[\label{subfig:complex-stylenerf} StyleNeRF\eyeglasses\yellowhat]{
        \includegraphics[width=0.31\linewidth]{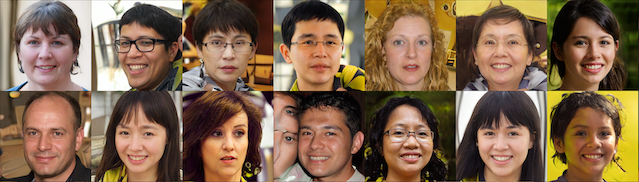}  
    }
    \subfigure[\label{subfig:complex-ldm-ddim} LDM-DDIM ($\eta=0$)\eyeglasses\yellowhat]{
        \includegraphics[width=0.31\linewidth]{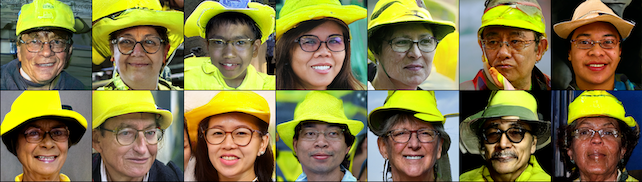}  
    }
    \subfigure[\label{subfig:complex-diffae} DiffAE\eyeglasses\yellowhat]{
        \includegraphics[width=0.31\linewidth]{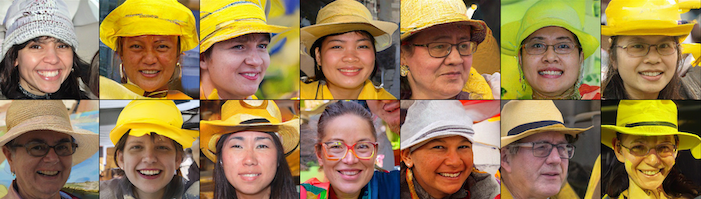}  
    }
    \subfigure[\label{subfig:ddgan-old} DDGAN\oldperson]{
        \includegraphics[width=0.31\linewidth]{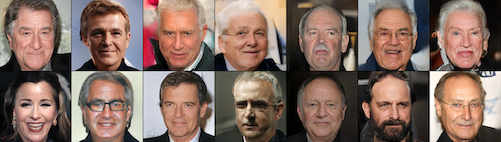}
    }
    \subfigure[\label{subfig:sn-ddpm-old} SN-DDPM\oldperson ]{
        \includegraphics[width=0.31\linewidth]{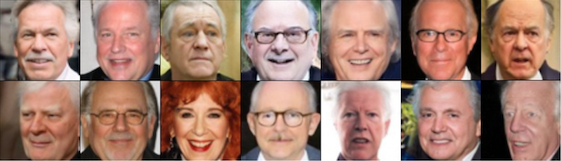}  
    }
    \subfigure[\label{subfig:sn-ddpm-eyeglasses} SN-DDPM\eyeglasses ]{
        \includegraphics[width=0.31\linewidth]{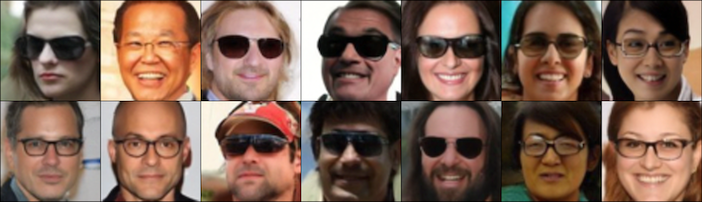}  
    }
\caption{\label{fig:clip-samples} Sampling sub-populations from pre-trained generative models. Notations follow Figure~\ref{fig:clip-curves}. 
}
\end{figure}

\begin{table}[!th]
    \caption{Guiding diffusion models and GANs with ID. ID-$\{$A, B, C, D$\}$ are images from FFHQ. The metric is the ArcFace cosine similarity \cite{Deng2019ArcFaceAA}. See samples in Figure~\ref{fig:individual} (Appendix~\ref{app:more-coverage}). }\smallskip
    \label{tab:id-results}
    \centering
    \begin{adjustbox}{width=\linewidth}
    \begin{tabular}{@{}lccccccc@{}}
        \toprule
        & \multicolumn{2}{c}{\textit{2D GAN}} & \multicolumn{2}{c}{\textit{3D GAN}} & \multicolumn{3}{c}{\textit{Diffusion model}}  \\
        \cmidrule(lr){2-3}\cmidrule(lr){4-5}\cmidrule(l){6-8}
                    & StyleGAN2 & StyleGAN-XL   & GIRAFFE-HD    & EG3D & LDM-DDIM ($\eta = 0$) & DDGAN & DiffAE \\
        \midrule
        ID-A        & 0.561     & 0.681         & 0.616         & 0.468 & \bf 0.904 & 0.837 & 0.873 \\
        ID-B        & 0.604     & 0.688         & 0.590         & 0.454 & \bf 0.896 & 0.805 & 0.838 \\
        ID-C        & 0.495     & 0.636         & 0.457         & 0.403 & \bf 0.892 & 0.795 & 0.852 \\
        ID-D        & 0.554     & 0.687         & 0.574         & 0.436 & \bf 0.911 & 0.831 & 0.873 \\
        \bottomrule
    \end{tabular}
    \end{adjustbox}
\end{table}

\section{Conclusions and Discussion}
\label{sec:conclusion}

This paper provides a unified view of pre-trained generative models by reformulating the latent space of diffusion models. While this reformulation is purely definitional, we show that it allows us to use diffusion models in similar ways as CycleGANs \cite{Zhu2017UnpairedIT} and GANs. Our \cyclediff achieves impressive performance on unpaired image-to-image translation (with two diffusion models trained on two domains independently) and zero-shot image-to-image translation (with text-to-image diffusion models). Our definition of latent code also allows diffusion models to be guided in the same way as GANs (i.e., plug-and-play, without finetuning on noisy images), and results show that diffusion models have broader coverage of sub-populations and individuals than GANs. 

Besides the interesting results, it is worth noting that this paper raised more questions than provided answers. We have provided a formal analysis of the common latent space of stochastic DPMs via the bounded distance between images (Section~\ref{subsec:cyclediff-method}), but it still needs further study. Notably, \citet{Khrulkov2022UnderstandingDL} and \citet{Su2022DualDI} studied deterministic DPMs based on optimal transport. 
Furthermore, \textit{efficient} plug-and-play guidance for stochastic DPMs on high-resolution images with many diffusion steps still remains open. These topics can be further explored in future studies. 

\bibliography{reference}
\bibliographystyle{iclr2023_conference}

\clearpage

\appendix

\section{Mathematical Details of Diffusion Models}
\label{app:math-details}

\subsection{Stochastic DPMs}
\label{subapp:ddpm-variants-math}

In Eq.~(\ref{eq:diffusion-as-implicit-latent}), we use $\boldmu_{T}(\boldx_{t}, t)$ and $\boldsigma_t$ as a high-level abstraction to represent each reverse step $t$ ($T$ is the total number of steps) of stochastic DPMs. In Eq.~(\ref{eq:diffusion-as-implicit-latent}), we define the sampling of $\imagecolor{\boldx}$ as
\begin{equation}
\label{eq:diffusion-as-implicit-latent-copy}
\begin{split}
    &\latentcolor{\boldz := \big(\boldx_T \oplus \boldeps_T \oplus \cdots \oplus \boldeps_{2} \oplus \boldeps_{1}\big)} \sim \mathcal{N}(\bm{0}, \bm{I}), \\
    &\boldx_{T-1} = \boldmu_{T}(\latentcolor{\boldx_{T}}, T) + \boldsigma_T \odot \latentcolor{\boldeps_T}, \\
	&\boldx_{t-1} = \boldmu_{T}(\boldx_{t}, t) + \boldsigma_t \odot \latentcolor{\boldeps_t}, \quad T > t > 0, \\ 
	&\imagecolor{\boldx := \boldx_0}.
\end{split}
\end{equation}
To be self-contained, here we provide details of $\boldmu_{T}(\boldx_{t}, t)$ and $\boldsigma_t$ for DDPM \cite{HoJA20} and DDIM \cite{song2021denoising}. Since the notations are not consistent in the two papers, we follow the notation in each paper respectively and use \ddpmcolor{different} \ddimcolor{colors} to distinguish different notations. Also note that $\boldeps_{\theta}\left(\boldx_{t}, t\right)$ stands for the neural network learned by DDPM and its variants, which should be distinguished from $\boldeps_t$ used throughout this paper. 

\noindent\textbf{DDPM's $\ddpmcolor{\boldmu_{T}}(\boldx_{t}, t)$ and $\ddpmcolor{\boldsigma_t}$:} 
We follow the notation in \citet{HoJA20}.
\begin{align}
    & \ddpmcolor{\boldmu_{T}}(\boldx_{t}, t) := \frac{1}{\sqrt{\ddpmcolor{\alpha_{t}}}}\left(\boldx_{t}-\frac{\ddpmcolor{\beta_{t}}}{\sqrt{1-\ddpmcolor{\bar{\alpha}_{t}}}} \boldeps_{\theta}\left(\boldx_{t}, t\right)\right), \\
    & \ddpmcolor{\boldsigma_t} := \left\{  
        \begin{array}{lr}  
             \sqrt{\ddpmcolor{\beta_t}}\bm{I}, & \text{(option 1)} \\  
             \displaystyle \sqrt{\frac{(1 - \ddpmcolor{\bar{\alpha}_{t-1}})\ddpmcolor{\beta_t}}{1 - \ddpmcolor{\bar{\alpha}_{t}}}}\bm{I}, & \text{(option 2)} \\  
             \displaystyle \exp\left(\frac{\bm{v}_{\theta}(\boldx_{t}, t)}{2} \log \ddpmcolor{\beta_t} + \frac{\bm{I} - \bm{v}_{\theta}(\boldx_{t}, t)}{2} \log \frac{(1 - \ddpmcolor{\bar{\alpha}_{t-1}})\ddpmcolor{\beta_t}}{1 - \ddpmcolor{\bar{\alpha}_{t}}}\right). & \text{(option 3)}
        \end{array}  \right.  
\end{align}

\noindent\textbf{DDIM's $\ddimcolor{\boldmu_{T}}(\boldx_{t}, t)$ and $\ddimcolor{\boldsigma_t}$:} 
We follow the notation in \citet{song2021denoising}.
\begin{align}
    & \ddimcolor{\boldmu_{T}}(\boldx_{t}, t) := \sqrt{\ddimcolor{\alpha_{t-1}}}\left(\frac{\boldx_{t}-\sqrt{1-\ddimcolor{\alpha_{t}}} \boldeps_{\theta}\left(\boldx_{t}, t\right)}{\sqrt{\ddimcolor{\alpha_{t}}}}\right) + \sqrt{1-\ddimcolor{\alpha_{t-1}}-\sigma_{t}^{2}} \cdot \boldeps_{\theta}\left(\boldx_{t}, t\right), \\
    & \ddimcolor{\boldsigma_t} := \sigma_t \bm{I}, \quad\quad \text{where } \sigma_t = \eta \sqrt{\left(1-\ddimcolor{\alpha_{t-1}}\right) /\left(1-\ddimcolor{\alpha_{t}}\right)} \sqrt{1-\ddimcolor{\alpha_{t}} / \ddimcolor{\alpha_{t-1}}},
\end{align}
where $\eta$ is a hyper-parameter. 

\subsection{Deterministic DDIM}
\noindent\textbf{Deterministic DDIM's $\ddimcolor{\boldmu_{T}}(\boldx_{t}, t)$:} 
Deterministic DDIM is a special case of DDIM when $\eta = 0$. For details of other deterministic DPMs, please check the original papers. 

\subsection{Score-based Generative Modeling with SDE}

\citet{0011SKKEP21} proposed a unified view of DDPM and score matching with Langevin dynamics (SMLD) as different stochastic differential equations (SDEs). Since the randomness in their sampling algorithms purely come from Gaussian noise, we can incorporate their models and sampling methods into our framework. As a demonstration, we show how to define $\boldmu_{T}(\boldx_{t}, t)$ and $\boldsigma_t$ for their predictor-only sampling with reverse diffusion samplers. Given a forward SDE:
\begin{equation}
    \mathrm{d}\boldx = \bm{f}(\boldx, t) \mathrm{d}t + \boldsigma(t) \odot \mathrm{d}\bm{w},
\end{equation}
the reverse-time SDE is 
\begin{equation}
    \mathrm{d}\boldx = [\bm{f}(\boldx, t) - \boldsigma(t)^2 \odot \nabla_{\boldx}\log p_t(\boldx)] \mathrm{d}t + \boldsigma(t) \odot \mathrm{d}\bar{\bm{w}}.
\end{equation}
Suppose the forward SDE is discretized in the following form:
\begin{equation}
    \boldx_{t+1} = \boldx_{t} + \bm{f}_t(\boldx_t) + \boldsigma_t \odot \boldz_t, \quad t = 0, \ldots, T - 1, \quad \boldz_t \sim \mathcal{N}(\bm{0}, \bm{I}). 
\end{equation}
Reverse diffusion samplers discretize the reserve-time SDE in a similar form:
\begin{equation}
\label{eq:scoresde-reverse-diffusion}
\begin{split}
    &\boldx_{t-1} = \boldx_{t} - \bm{f}_{t}(\boldx_{t}) + \boldsigma_{t}^{2} \odot \bm{s}_{\boldtheta}(\boldx_{t}, t) + \boldsigma_{t} \odot \boldeps_{t}, \quad t = 1, \ldots, T, \quad \boldeps_{t} \sim \mathcal{N}(\bm{0}, \bm{I}),
\end{split}
\end{equation}
where $\bm{s}_{\boldtheta}$ is a neural network trained to match the score $\nabla_{\boldx}\log p_t(\boldx)$. By comparing Eq.~(\ref{eq:diffusion-as-implicit-latent-copy}) and Eq.~(\ref{eq:scoresde-reverse-diffusion}), we have $\boldmu_{T}(\boldx_{t}, t) := \boldx_{t} - \bm{f}_{t}(\boldx_{t}) + \boldsigma_{t}^{2} \odot \bm{s}_{\boldtheta}(\boldx_{t}, t)$. 

\subsection{DDGAN}

In Eq.~(\ref{eq:ddgan}), we use $\boldmu_{T}(\boldx_{t}, \boldz_t, t)$ and $\boldsigma_t$ as high-level abstractions to represent each reverse step $t$ ($T$ is the total number of steps) of DDGAN. The generation process is defined as
\begin{equation}
\begin{split}
    &\latentcolor{\boldz := \big(\boldx_T \oplus \boldz_T \oplus \boldeps_T \oplus \cdots \oplus \boldz_{2} \oplus \boldeps_{2} \oplus \boldz_{1}\big)} \sim \mathcal{N}(\bm{0}, \bm{I}), \\
    & \boldx_{T-1} = \boldmu_{T}(\latentcolor{\boldx_{T}}, \latentcolor{\boldz_T}, T) + \boldsigma_T \odot \latentcolor{\boldeps_{T}}, \\
    & \boldx_{t-1} = \boldmu_{T}(\boldx_{t}, \latentcolor{\boldz_t}, t) + \boldsigma_t \odot \latentcolor{\boldeps_{t}}, \quad T > t > 1, \\
    & \imagecolor{\boldx := \boldx_{0}} = \boldmu_{T}(\boldx_{1}, \latentcolor{\boldz_1}, 1).
\end{split}
\end{equation}
To be self-contained, here we provide details of $\boldmu_{T}(\boldx_{t}, \boldz_t, t)$ and $\boldsigma_t$ of DDGAN. 

\noindent\textbf{DDGAN's $\boldmu_{T}(\boldx_{t}, \boldz_t, t)$ and $\boldsigma_t$: } We follow the notation in \citet{xiao2022tackling} and \citet{HoJA20}.
\begin{align}
    & \boldmu_{T}(\boldx_{t}, \boldz_t, t) := \frac{\sqrt{\ddpmcolor{\bar{\alpha}_{t-1}}} \ddpmcolor{\beta_{t}}}{1-\ddpmcolor{\bar{\alpha}_{t}}} G_\theta(\boldx_{t}, \boldz_t, t)+\frac{\sqrt{\ddpmcolor{\alpha_{t}}}\left(1-\ddpmcolor{\bar{\alpha}_{t-1}}\right)}{1-\ddpmcolor{\bar{\alpha}_{t}}} \boldx_{t}, \\
    & \boldsigma_t := \sqrt{\frac{(1 - \ddpmcolor{\bar{\alpha}_{t-1}})\ddpmcolor{\beta_t}}{1 - \ddpmcolor{\bar{\alpha}_{t}}}}\bm{I},
\end{align}
where $G_\theta(\boldx_{t}, \boldz_t, t)$ is a conditional GAN learned by DDGAN, which should be distinguished from the deterministic mapping $\gen$ used throughout this paper. 

\begin{algorithm}[!ht]
\DontPrintSemicolon
\textbf{Input:} an image $\imagecolor{\boldx := \boldx_0}$, a pre-trained stochastic DPM with $\boldmu_{T}(\boldx_{t}, t)$, $\boldsigma_t$, and $q(\boldx_{1:T}|\imagecolor{\boldx_0})$\;
1. Sample $\boldx_1, \ldots, \boldx_{T-1}, \latentcolor{\boldx_T} \sim q(\boldx_{1:T}|\imagecolor{\boldx_0})$\;
2. $\latentcolor{\boldz} = \latentcolor{\boldx_T}$\;
	\For{$t = T, \ldots, 1$}{		
		3. $\latentcolor{\boldeps_t} = \big(\boldx_{t-1} - \boldmu_{T}(\boldx_{t}, t)\big) / \boldsigma_t$\;
		4. $\latentcolor{\boldz} = \latentcolor{\boldz} \oplus \latentcolor{\boldeps_t}$\;
	}
5. \textbf{Output: } $\latentcolor{\boldz}$\;
	\caption{DPM-Encoder}
	\label{alg:dpm-encoder-copy}
\end{algorithm}

\section{Mathematical Details of DPM-Encoder}
\label{app:dpm-encoder}

In this section, we provide details of our DPM-Encoder introduced in Section~\ref{subsec:encoding}, which samples $\latentcolor{\boldz} \sim \mathrm{DPMEnc}(\latentcolor{\boldz}|\imagecolor{\boldx}, G)$. For each image \imagecolor{$\boldx := \boldx_0$}, stochastic DPMs define a posterior distribution over the noisy images $\boldx_{1:T}$, denoted as $q(\boldx_{1:T}|\imagecolor{\boldx_0})$ \cite{HoJA20,song2021denoising}. To be self-contained, we provide details of this posterior distribution for different diffusion models. 

\noindent\textbf{DDPM's posterior $\ddpmcolor{q}(\boldx_{1:T}|\imagecolor{\boldx_0})$: } We follow the notation in \citet{HoJA20}.
\begin{equation}
    \ddpmcolor{q}(\boldx_{1: T} | \imagecolor{\boldx_{0}}):=\prod_{t=1}^{T} \ddpmcolor{q}(\boldx_{t} | \boldx_{t-1}), \quad \ddpmcolor{q}(\boldx_{t} | \boldx_{t-1}):=\mathcal{N}\left(\boldx_{t} ; \sqrt{1-\ddpmcolor{\beta_{t}}} \boldx_{t-1}, \ddpmcolor{\beta_{t}} \mathbf{I}\right).
\end{equation}

\noindent\textbf{DDIM's posterior $\ddimcolor{q}(\boldx_{1:T}|\imagecolor{\boldx_0})$: } We follow the notation in \citet{song2021denoising}. 
\begin{align}
    & \ddimcolor{q}(\boldx_{1: T} | \imagecolor{\boldx_{0}}):=\ddimcolor{q}(\boldx_{T} | \imagecolor{\boldx_{0}}) \prod_{t=2}^{T} \ddimcolor{q}(\boldx_{t-1} | \boldx_{t}, \imagecolor{\boldx_{0}}), \\
    & \ddimcolor{q}(\boldx_{T} | \imagecolor{\boldx_{0}})=\mathcal{N}(\sqrt{\ddimcolor{\alpha_{T}}} \imagecolor{\boldx_{0}},\left(1-\ddimcolor{\alpha_{T}}\right) \boldsymbol{I}), \\
    \begin{split}
        & \ddimcolor{q}(\boldx_{t-1} | \boldx_{t}, \imagecolor{\boldx_{0}})=\mathcal{N}\left(\sqrt{\ddimcolor{\alpha_{t-1}}} \imagecolor{\boldx_{0}}+\sqrt{1-\ddimcolor{\alpha_{t-1}}-\sigma_{t}^{2}} \cdot \frac{\boldx_{t}-\sqrt{\ddimcolor{\alpha_{t}}} \imagecolor{\boldx_{0}}}{\sqrt{1-\ddimcolor{\alpha_{t}}}}, \sigma_{t}^{2} \boldsymbol{I}\right), \\
        & \quad\quad\quad\quad \text{where } \sigma_t = \eta \sqrt{\left(1-\ddimcolor{\alpha_{t-1}}\right) /\left(1-\ddimcolor{\alpha_{t}}\right)} \sqrt{1-\ddimcolor{\alpha_{t}} / \ddimcolor{\alpha_{t-1}}}.
    \end{split}
\end{align}

Based on the posterior distribution $q(\boldx_{1:T}|\imagecolor{\boldx_0})$, DPM-Encoder samples the latent code $\latentcolor{\boldz}$ by first sampling noisy images $\boldx_1, \ldots, \latentcolor{\boldx_T}$ from $q(\boldx_{1:T}|\imagecolor{\boldx_0})$ and computing the $\latentcolor{\boldeps_t}$ according to Eq.~(\ref{eq:diffusion-as-implicit-latent}) and Eq.~(\ref{eq:diffusion-as-implicit-latent-copy}). Formally, we define the sampling process $\latentcolor{\boldz} \sim \mathrm{DPMEnc}(\latentcolor{\boldz}|\imagecolor{\boldx}, G)$ as
\begin{equation}
\label{eq:dpm-encoder-copy}
\begin{split}
    &\boldx_1, \ldots, \boldx_{T-1}, \latentcolor{\boldx_T} \sim q(\boldx_{1:T}|\imagecolor{\boldx_0}), \quad \latentcolor{\boldeps_t} = \big(\boldx_{t-1} - \boldmu_{T}(\boldx_{t}, t)\big) / \boldsigma_t, \quad t = T, \ldots, 1, \\
    &\quad\quad\quad\quad\quad\quad\quad\quad\quad\quad  \latentcolor{\boldz := \big(\boldx_T \oplus \boldeps_T \oplus \cdots \oplus \boldeps_{2} \oplus \boldeps_{1}\big)}.
\end{split}
\end{equation}
DPM-Encoder guarantees perfect reconstruction. The proof is straightforward, provided as follows. 

\newtheorem{prop}{Proposition}
\begin{prop}
\label{prep:latent-space-ebm}
    (Invertibility of DPM-Encoder) For each $\latentcolor{\boldz} \sim \mathrm{DPMEnc}(\latentcolor{\boldz}|\imagecolor{\boldx}, G)$ defined in Eq.~(\ref{eq:dpm-encoder-copy}), we have $\imagecolor{\boldx} = \imagecolor{\bar{\boldx}} := G(\latentcolor{\boldz})$, where $\imagecolor{\bar{\boldx}} := G(\latentcolor{\boldz})$ is defined as
    \begin{equation}
    \begin{split}
        &\bar{\boldx}_{T-1} = \boldmu_{T}(\latentcolor{\boldx_{T}}, T) + \boldsigma_T \odot \latentcolor{\boldeps_T}, \\
	    &\bar{\boldx}_{t-1} = \boldmu_{T}(\bar{\boldx}_{t}, t) + \boldsigma_t \odot \latentcolor{\boldeps_t}, \quad T > t > 0, \\
	    & \imagecolor{\bar{\boldx} := \bar{\boldx}_0}.
    \end{split}
    \end{equation}
\end{prop}

\begin{proof}
    We prove $\bar{\boldx}_{t} = \boldx_{t}$ for all $T - 1 \geq t \geq 0$ by induction. The proposition holds when $\imagecolor{\bar{\boldx}_{0}} = \imagecolor{\boldx_{0}}$. To begin with, $\bar{\boldx}_{T-1} = \boldx_{T-1}$ because
    \begin{align}
        \bar{\boldx}_{T-1} &= \boldmu_{T}(\latentcolor{\boldx_{T}}, T) + \boldsigma_T \odot \latentcolor{\boldeps_T} \\
        &= \boldmu_{T}(\latentcolor{\boldx_{T}}, T) + \boldsigma_T \odot \big(\boldx_{T-1} - \boldmu_{T}(\latentcolor{\boldx_{T}}, T)\big) / \boldsigma_T = \boldx_{T-1}.
    \end{align}
    For $T - 1 \geq t > 0$, when $\bar{\boldx}_{t} = \boldx_{t}$, we have
    \begin{align}
        \bar{\boldx}_{t-1} &= \boldmu_{T}(\bar{\boldx}_t, t) + \boldsigma_t \odot \latentcolor{\boldeps_t} \\
        &= \boldmu_{T}(\boldx_t, t) + \boldsigma_t \odot \latentcolor{\boldeps_t} \\
        &= \boldmu_{T}(\boldx_t, t) + \boldsigma_t \odot \big(\boldx_{t-1} - \boldmu_{T}(\boldx_{t}, t)\big) / \boldsigma_t = \boldx_{t-1}.
    \end{align}
\end{proof}

\section{Experimental Details of Zero-Shot Image-to-Image Translation}
\label{app:experimental-details}

\noindent\textbf{Sources of images in the 150 tuples: } \ For the zero-shot image-to-image translation experiment, we created a set of 150 tuples as task input, which include but are not limited to: (1) image generated by \dalle 2 \cite{Ramesh2022HierarchicalTI}, (2) real images from  \citet{Ruiz2022DreamBoothFT}, (3) real images from \cite{Hertz2022PrompttoPromptIE}, (4) real images collected by the authors. 

\noindent\textbf{Per sample selection criterion: } For each test sample, we allow each method to enumerate some combinations of hyperparameters (detailed below). To select the best combination for each sample, we used the directional CLIP score $\mathcal{S}_{\text{D-CLIP}}$ as the criterion (higher is better). 

\noindent\textbf{DDIB: } DDIB edits images by using a deterministic DPM conditioned on the source text $\bm{t}$ to encode the source image, followed by decoding conditioned on the target text $\hat{\bm{t}}$. We used the deterministic DDIM sampler with 100 steps. We set the classifier-free guidance of the encoding step as $1$; we enumerated the classifier-free guidance of the decoding step as $\{1, 1.5, 2, 3, 4, 5\}$. 

\noindent\textbf{SDEdit: } SDEdit edits images by adding noise to the original image (the encoding step), followed by denoising the noised image with a diffusion model trained on the target domain (the decoding step). For zero-shot image-to-image translation, the decoding step of SDEdit uses the text-to-image diffusion model conditioned on the target image $\hat{\bm{t}}$. Notably, SDEdit does not provide a way to take the source text $\bm{t}$ as input. We used the DDIM sampler ($\eta = 0.1$) with 100 steps. We enumerated the classifier-free guidance of the decoding step as $\{1, 1.5, 2, 3, 4, 5\}$; we enumerated the encoding step as $\{15, 20, 25, 30, 40, 50\}$; we ran $15$ trials for each hyperparameter combination. 

\noindent\textbf{\cyclediff: } For our \cyclediff, we used the DDIM sampler ($\eta = 0.1$) with 100 steps. We set the classifier-free guidance of the encoding process as $1$; we enumerated the classifier-free guidance of the decoding step as $\{1, 1.5, 2, 3, 4, 5\}$; we enumerated the early stopping step $T_{\text{es}}$ as $\{15, 20, 25, 30, 40, 50\}$; we ran $15$ trials for each hyperparameter combination. 

\section{Resources}

Our experiments used publicly available pre-trained checkpoints (except for the diffusion models trained by us on AFHQ Cat and Wild; see Section~\ref{sec:experiments}). Each experiment was run on one NVIDIA RTX A4000 (16G) / RTX A6000 (48G) / A100 (40G) GPU. \acceptedtext{Our codes are based on the PyTorch library and are now available at \opensourceunify~and \opensourcezero.}\anonymoustext{We will make our code, configuration files, and experiment commands publicly available. }

\clearpage

\section{Additional Results for Zero-Shot Image-to-Image Translation}
\label{app:dalle-details}

Figure~\ref{fig:dalle-experiment-baseline} provides a qualitative comparison for zero-shot image-to-image translation. Compared with DDIB and SDEdit, \cyclediff greatly improves the faithfulness to the source image.

\begin{figure}[!th]
\centering
    \includegraphics[width=0.8\linewidth]{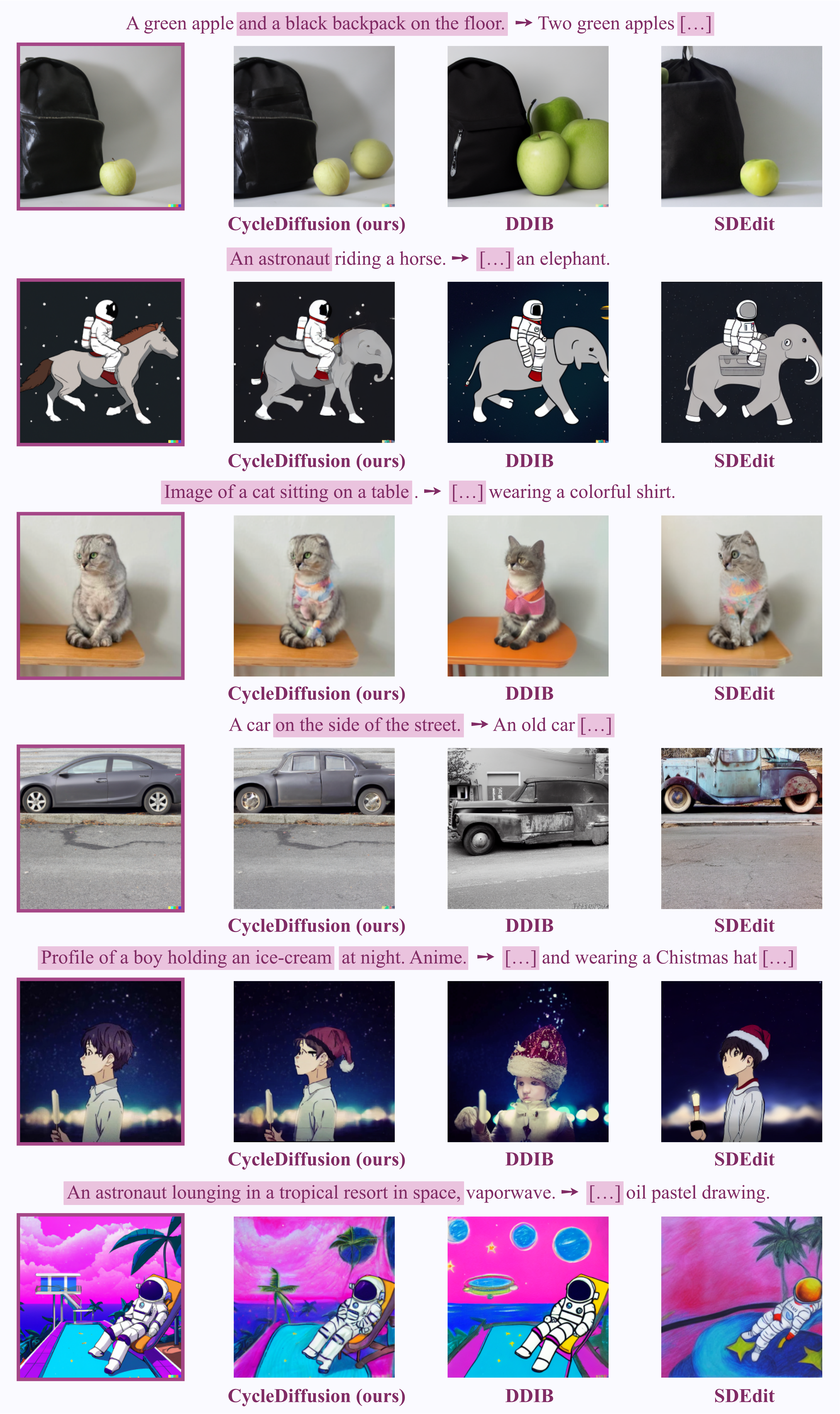}
\caption{\label{fig:dalle-experiment-baseline} Samples for zero-shot image-to-image translation. Notations follow Figure~\ref{fig:dalle-experiment}. Compared with DDIB and SDEdit, \cyclediff greatly improves the faithfulness to the source image. }
\end{figure}

\section{Local Editing DDIM's High-Dimensional Latent Code}
\label{subsec:local-editing} 

Local editing of low-dimensional latent code has been shown to be useful for semantic-level image manipulation \cite{Shen2022InterFaceGANIT}. However, it is unclear whether we can perform semantic-level image manipulation via local editing in the high-dimensional latent space diffusion models. Note that this is different from mask-then-inpaint \cite{Ramesh2022HierarchicalTI}, edit-with-scribbles \cite{meng2022sdedit}, or domain adaptation \cite{Kim2021DiffusionCLIPTD}). Notably, it does not need the classifier to be adapted to noisy images as done by the classifier guidance \cite{dhariwal2021diffusion,Liu2021MoreCF}.

Given an image $\imagecolor{\boldx_{\itm{ori}}}$, we encode it as $\latentcolor{\boldz_{\itm{ori}}}$, edit it as $\latentcolor{\boldz_{\itm{edit}}} = \latentcolor{\boldz_{\itm{ori}}} + \latentcolor{\bm{n}}$, and compute the edited image $\imagecolor{\boldx_{\itm{edit}}} = \gen(\latentcolor{\boldz_{\itm{edit}}})$. 
To learn the vector $\latentcolor{\bm{n}}$ for a target class $a$, we optimize
\begin{equation}
    \mathop{\arg\min}_{\|\latentcolor{\bm{n}}\|_2 = r} \mathbb{E}_{\latentcolor{\boldz_{\itm{ori}}} \sim p_{\latentcolor{\boldz}}(\latentcolor{\boldz_{\itm{ori}}}), \ \latentcolor{\boldz_{\itm{edit}}} = \latentcolor{\boldz_{\itm{ori}}} + \latentcolor{\bm{n}}}\Big[- \lambda_{\itm{cls}}\log P\big(a|\gen(\latentcolor{\boldz_{\itm{edit}}})\big) - \cos\big\langle R(\gen(\latentcolor{\boldz_{\itm{edit}}})), R(\gen(\latentcolor{\boldz_{\itm{ori}}}))\big\rangle  \Big],
\end{equation}
where $P(\cdot|\imagecolor{\boldx})$ is a classifier trained on CelebA \cite{Liu2015DeepLF}, and $R$ is the IR-SE50 face embedding model \cite{Deng2019ArcFaceAA} to preserve the identity. Empirically, we find that LDM-DDIM ($\eta=0$) works the best for local editing, as shown in Figure~\ref{fig:local-editing}.

\begin{figure}[!th]
\centering
    \includegraphics[width=\linewidth]{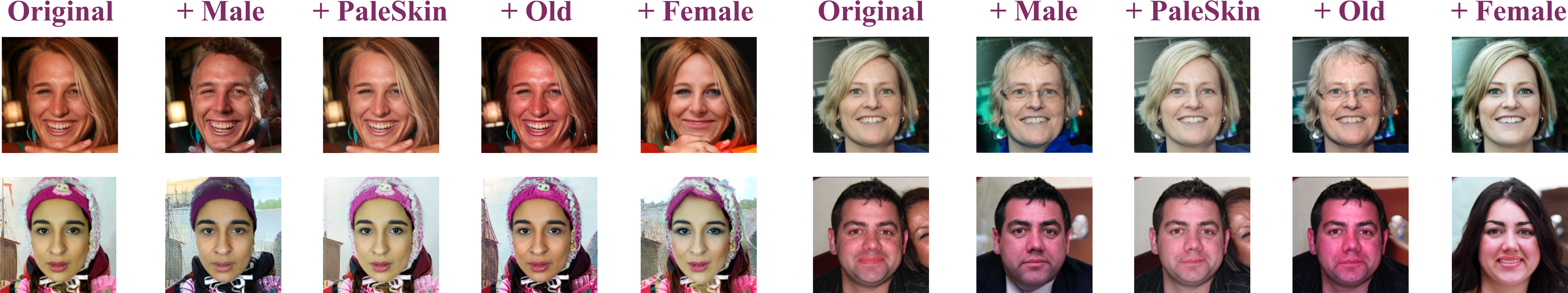}
\caption{\label{fig:local-editing} Image manipulation by local editing of diffusion models' latent code. The diffusion model used here is the deterministic LDM-DDIM ($\eta=0$). }
\end{figure}

\begin{table}[!ht]
	\centering
	\caption{State-of-the-art generative models used in this paper. Notations: \textit{struc.}, --, and $\oplus$ stand for \textit{structure}, \textit{no ``latent codes''},\footref{footnote:latent-code} and \textit{progressive generation}, respectively. 
	}\smallskip
	\label{tab:models-included}
	\begin{adjustbox}{width=\linewidth}
		\begin{tabular}{@{}l@{}l@{}c@{}c@{}c@{}c@{}c@{}}
			\toprule
			& \bf  Model name & \bf Latent prior & \bf Objective & \bf  Architecture & \ \ \ \ \bf Latent struc. \ \ \ \ & \bf Resolution \\
			\midrule
			\multirow{7}*{\textit{Diffusion} \ \ } 
			& DDPM \cite{HoJA20} \textit{etc.}          & -- & ELBO & \multirow{7}*{$\{\text{CNN}, \text{ViT}\}$} & -- & 256 \\
			& DDIM ($\eta=0$) \cite{song2021denoising}  & Gaussian  & ELBO & & spatial & 256 \\
			& SN-DDPM \cite{Bao2022EstimatingTO}        & -- & ELBO & & -- & 64 \\
			& ScoreSDE \cite{0011SKKEP21}               & --        & \ \ ELBO / SM \ \ & & -- & 256 / 1024    \\
            & LDM \cite{Rombach2021ldm} \textit{etc.}   & diffusion & ELBO & & spatial & 256 \\
            & DiffAE \cite{Preechakul2021DiffusionAT}   & diffusion & ELBO & & hybrid & 256 \\
            & DDGAN \cite{xiao2022tackling}             & -- & hybrid & & -- & 256 \\
			\midrule
	        \multirow{5}*{\textit{2D GAN}} 
	        & StyleGAN2 \cite{Karras2020AnalyzingAI}    & \multirow{5}*{Gaussian} & GAN & CNN & \multirow{5}*{vector} & 1024 \\
            & StyleGAN-XL \cite{Sauer2022StyleGANXLSS}  &  & GAN & CNN          & & 256 -- 1024 \\
            & StyleSwin \cite{Zhang2021StyleSwinTG}     &  & GAN& ViT          & & 256 / 1024 \\
            & BigGAN \cite{brock2018large}              &  & GAN & CNN          & & 256 \\
            & Diffusion-GAN \cite{Wang2022DiffusionGANTG} \  &  & hybrid & CNN          & & 1024 \\
            \midrule
	        \multirow{4}*{\textit{3D GAN}} 
	        & StyleNeRF \cite{gu2022stylenerf}          & \multirow{4}*{Gaussian} & \multirow{4}*{GAN} & NeRF $\oplus$ CNN   & \multirow{4}*{vector} & 256 -- 1024 \\
            & GIRAFFE-HD \cite{Xue2022GIRAFFEHA}        &  & & NeRF $\oplus$ CNN    & & 1024 \\
            & StyleSDF \cite{OrEl2021StyleSDFH3}        &  & & SDF $\oplus$ CNN    & & 512 / 1024 \\
            & EG3D \cite{Chan2021EfficientG3}           &  & & TriPl $\oplus$ CNN    & & 512 \\
			\midrule
			\multirow{1}*{\textit{VAE}} 
			& NVAE \cite{VahdatK20} & Gaussian & ELBO & CNN & spatial & 256 \\
			\bottomrule
		\end{tabular}
	\end{adjustbox}
\end{table}

\section{Additional Results for Plug-and-Play Guidance}
\label{app:more-coverage}

Seen in Table~\ref{tab:models-included} is a summary of generative models unified as deterministic mappings in this paper. Different models have different training objectives, model architectures, and structures of ``latent code''\footref{footnote:latent-code}. Most of the listed models are included in our experiments. Table~\ref{tab:fine-clip-1} and Table~\ref{tab:fine-clip-2} provide a more detailed version of the results (for some generative models) seen in Figure~\ref{fig:clip-curves}. Specifically, we investigated different configurations of various diffusion models and GANs. In Figure~\ref{fig:individual}, we provide several image samples for ID-controlled sampling from pre-trained generative models. Consistent with Table~\ref{tab:id-results}, diffusion models have better coverage of individuals than 2D/3D GANs. 

\nocite{Cao2022ASO}

\nocite{Yang2022DiffusionMA}

\nocite{Croitoru2022DiffusionMI}

\begin{table}[!th]
    \caption{CLIP experiments of models that have different configurations. Numbers under each model stand for the image resolution; \textit{trunc.} $\phi=0.7$ stands for the truncation trick \cite{Karras2019ASG} with truncation coefficient $\phi=0.7$. The reported metric is the CLIP score (larger is better), the same as Figure~\ref{fig:clip-curves}. $^\heartsuit$ and $^\spadesuit$ stand for the configuration plotted in Figure~\ref{fig:clip-curves}. }\smallskip
    \label{tab:fine-clip-1}
    \centering
    \begin{adjustbox}{width=\linewidth}
    \begin{tabular}{@{}lcccccccccc@{}}
        \toprule
        \multicolumn{1}{c}{Text $\controlcolor{\bm{t}}$ (Figure~\ref{fig:clip-curves})} & \multicolumn{5}{c}{\baby$^\heartsuit$} & \multicolumn{5}{c}{\oldperson$^\spadesuit$} \\
        \cmidrule(r){1-1}\cmidrule(lr){2-6}\cmidrule(l){7-11}
        \multicolumn{1}{c}{Control strength $\lambda_{\text{CLIP}}$} & 100 & 300 & 500 & 700 & 1000 & 100 & 300 & 500 & 700 & 1000 \\
        \midrule
        LDM-DDIM ($\eta = 0$) \\
        \quad 256 ($T_{\textit{g}} = 10$)$^{\heartsuit\spadesuit}$ & 0.258 	& 0.276 	& 0.283 	& 0.288 	& 0.290  & 0.269 	& 0.283 	& 0.296 	& 0.300     & 0.308 \\
        \quad 256 ($T_{\textit{g}} = 5$) & 0.257 	& 0.280 	& 0.283 	& 0.285 	& 0.287  & 0.269 	& 0.283 	& 0.292 	& 0.298 	& 0.304  \\
        DiffAE \\
        \quad 256 ($T_{\textit{g}} = 10$)$^{\heartsuit\spadesuit}$ & 0.266 	& 0.287 	& 0.291 	& 0.294 	& 0.294 	& 0.270 	& 0.296 	& 0.307 	& 0.314 	& 0.319  \\
        \quad 128 ($T_{\textit{g}} = 3$) & 0.259 	& 0.284 	& 0.289 	& 0.290 	& 0.292 	& 0.256 	& 0.271 	& 0.286 	& 0.293 	& 0.298  \\
        \quad 128 ($T_{\textit{g}} = 3$, $\latentcolor{\boldz_T}$ only) & 0.256 	& 0.289 	& 0.289 	& 0.290 	& 0.295 	& 0.256 	& 0.270 	& 0.285 	& 0.293 	& 0.297 \\
        StyleGAN2 \\
        \quad 1024$^{\heartsuit\spadesuit}$ & 0.273 		& 0.293 		& 0.296 		& 0.296 		& 0.298 		& 0.275 		& 0.302 		& 0.308 		& 0.311 		& 0.312 \\
        \quad 1024 (\textit{trunc.} $\phi=0.7$)      & 0.267 		& 0.287 		& 0.291 		& 0.293 		& 0.293 		& 0.267 		& 0.291 		& 0.299 		& 0.301 		& 0.303 \\
        StyleGAN-XL \\
        \quad 1024$^{\heartsuit\spadesuit}$ & 0.270		& 0.291		& 0.294		& 0.295		& 0.295    & 0.273 		& 0.299 		& 0.308 		& 0.312 		& 0.313 \\
        \quad 1024 (\textit{trunc.} $\phi=0.7$)      & 0.263 		& 0.283 		& 0.287 		& 0.289 		& 0.290 		& 0.265 		& 0.284 		& 0.292 		& 0.295 		& 0.297 \\
        \quad 512 (\textit{trunc.} $\phi=0.7$)       & 0.263 		& 0.282 		& 0.286 		& 0.288 		& 0.289 		& 0.263 		& 0.284 		& 0.293 		& 0.296 		& 0.300 \\
        \quad 256 (\textit{trunc.} $\phi=0.7$)       & 0.262 		& 0.281 		& 0.284 		& 0.287 		& 0.289 		& 0.259 		& 0.281 		& 0.291 		& 0.295 		& 0.299 \\
        StyleSwin \\
        \quad 1024$^{\heartsuit\spadesuit}$ & 0.266 		& 0.279 		& 0.278 		& 0.276 		& 0.268 		& 0.273 		& 0.291 		& 0.296 		& 0.295 		& 0.294 \\
        \quad 256                           & 0.262 		& 0.282 		& 0.283 		& 0.283 		& 0.281 		& 0.267 		& 0.285 		& 0.290 		& 0.293 		& 0.293 \\
        \quad 1024 (\textit{trunc.} $\phi=0.7$)      & 0.265 		& 0.284 		& 0.287 		& 0.288 		& 0.288 		& 0.264 		& 0.279 		& 0.292 		& 0.297 		& 0.300 \\
        \quad 256 (\textit{trunc.} $\phi=0.7$)       & 0.259 		& 0.278 		& 0.281 		& 0.275 		& 0.273 		& 0.261 		& 0.276 		& 0.281 		& 0.284 		& 0.281 \\
        Diffusion-GAN \\
        \quad 1024$^{\heartsuit\spadesuit}$ & 0.278 		& 0.295 		& 0.298 		& 0.298 		& 0.299 		& 0.270 		& 0.297 		& 0.305 		& 0.307 		& 0.308 \\
        \quad 1024 (\textit{trunc.} $\phi=0.7$)      & 0.273 		& 0.294 		& 0.294 		& 0.300 		& 0.301 		& 0.262 		& 0.286 		& 0.300 		& 0.305 		& 0.309 \\
        StyleNeRF \\
        \quad 1024 & 0.246 		& 0.271 		& 0.283 		& 0.288 		& 0.291 		& 0.264 		& 0.291 		& 0.303 		& 0.307 		& 0.311 \\
        \quad 256                           & 0.234 		& 0.240 		& 0.247 		& 0.252 		& 0.259 		& 0.260 		& 0.275 		& 0.291 		& 0.298 		& 0.303 \\
        \quad 1024 (\textit{trunc.} $\phi=0.7$)      & 0.243 		& 0.266 		& 0.277 		& 0.283 		& 0.287 		& 0.260 		& 0.280 		& 0.291 		& 0.296 		& 0.300 \\
        \quad 256 (\textit{trunc.} $\phi=0.7$)       & 0.229 		& 0.235 		& 0.239 		& 0.243 		& 0.249 		& 0.255 		& 0.267 		& 0.282 		& 0.290 		& 0.295 \\
        StyleSDF \\
        \quad 1024$^{\heartsuit\spadesuit}$ & 0.275 		& 0.288 		& 0.286 		& 0.282 		& --    		& 0.270 		& 0.290 		& 0.292 		& 0.291 		& -- \\
        \quad 1024 (\textit{trunc.} $\phi=0.7$)      & 0.267 		& 0.283 		& 0.284 		& 0.279 		& --    		& 0.261 		& 0.270 		& 0.273 		& 0.273 		& -- \\ 
        EG3D \\
        \quad 512$^{\heartsuit\spadesuit}$  & 0.277 		& 0.292 		& 0.294 		& 0.295 		& 0.294 		& 0.272 		& 0.298 		& 0.305 		& 0.307 		& 0.310 \\
        \quad 512 (\textit{trunc.} $\phi=0.7$)       & 0.277 		& 0.270 		& 0.276 		& 0.280 		& 0.283 		& 0.265 		& 0.285 		& 0.293 		& 0.296 		& 0.298 \\
        \bottomrule
    \end{tabular}
    \end{adjustbox}
\end{table}

\begin{figure}[!th]
\centering
    \includegraphics[width=\linewidth]{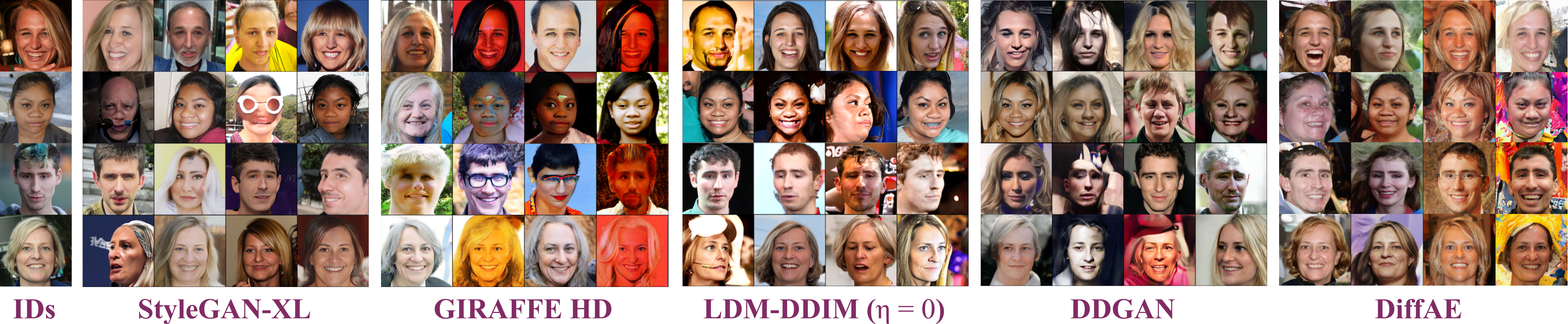}
\caption{\label{fig:individual} Image samples for the face ID experiment in Table~\ref{tab:id-results}.} 
\end{figure}

\clearpage

\begin{table}[!th]
    \caption{CLIP experiments of models that have different configurations. Numbers under each model stand for the image resolution; \textit{trunc.} $\phi=0.7$ stands for the truncation trick \cite{Karras2019ASG} with truncation coefficient $\phi=0.7$. The reported metric is the CLIP score (larger is better), the same as Figure~\ref{fig:clip-curves}. $^\heartsuit$ and $^\spadesuit$ stand for the configuration plotted in Figure~\ref{fig:clip-curves}. }\smallskip
    \label{tab:fine-clip-2}
    \centering
    \begin{adjustbox}{width=\linewidth}
    \begin{tabular}{@{}lcccccccccc@{}}
        \toprule
        \multicolumn{1}{c}{Text $\controlcolor{\bm{t}}$ (Figure~\ref{fig:clip-curves})} & \multicolumn{5}{c}{\eyeglasses$^\heartsuit$} & \multicolumn{5}{c}{\eyeglasses\yellowhat$^\spadesuit$} \\
        \cmidrule(r){1-1}\cmidrule(lr){2-6}\cmidrule(l){7-11}
        \multicolumn{1}{c}{Control strength $\lambda_{\text{CLIP}}$} & 100 & 300 & 500 & 700 & 1000 & 100 & 300 & 500 & 700 & 1000 \\
        \midrule
        LDM-DDIM ($\eta = 0$) \\
        \quad 256 ($T_{\textit{g}} = 10$)$^{\heartsuit\spadesuit}$ & 0.275 	& 0.290 	& 0.297 	& 0.301 	& 0.307 	& 0.252 	& 0.288 	& 0.312 	& 0.326 	& 0.343  \\
        \quad 256 ($T_{\textit{g}} = 5$) & 0.273 	& 0.289 	& 0.300 	& 0.305 	& 0.310 	& 0.250 	& 0.288 	& 0.315 	& 0.329 	& 0.343 \\
        DiffAE \\
        \quad 256 ($T_{\textit{g}} = 10$)$^{\heartsuit\spadesuit}$ & 0.275 	& 0.290 	& 0.297 	& 0.303 	& 0.307 	& 0.250 	& 0.287 	& 0.308 	& 0.320 	& 0.328 \\
        \quad 128 ($T_{\textit{g}} = 3$) & 0.265 	& 0.281 	& 0.288 	& 0.293 	& 0.298 	& 0.240 	& 0.273 	& 0.300 	& 0.314 	& 0.326 \\
        \quad 128 ($T_{\textit{g}} = 3$, $\latentcolor{\boldz_T}$ only) & 0.263 	& 0.280 	& 0.288 	& 0.291 	& 0.296 	& 0.240 	& 0.275 	& 0.300 	& 0.313 	& 0.324  \\
        StyleGAN2 \\
        \quad 1024$^{\heartsuit\spadesuit}$ & 0.279 		& 0.294 		& 0.300 		& 0.303 		& 0.304 		& 0.264 		& 0.299 		& 0.313 		& 0.319 		& 0.321 \\
        \quad 1024 (\textit{trunc.} $\phi=0.7$)      & 0.278 		& 0.291 		& 0.297 		& 0.300 		& 0.303 		& 0.255 		& 0.286 		& 0.303 		& 0.311 		& 0.316 \\
        StyleGAN-XL        & \\
        \quad 1024$^{\heartsuit\spadesuit}$ & 0.282 		& 0.299 		& 0.306 		& 0.310 		& 0.310 		& 0.260 		& 0.296 		& 0.301 		& 0.304 		& 0.305  \\
        \quad 1024 (\textit{trunc.} $\phi=0.7$)      & 0.281 		& 0.297 		& 0.303 		& 0.305 		& 0.309 		& 0.253 		& 0.279 		& 0.287 		& 0.288 		& 0.291 \\
        \quad 512 (\textit{trunc.} $\phi=0.7$)       & 0.280 		& 0.296 		& 0.301 		& 0.305 		& 0.307 		& 0.250 		& 0.282 		& 0.296 		& 0.300 		& 0.303 \\
        \quad 256 (\textit{trunc.} $\phi=0.7$)       & 0.275 		& 0.290 		& 0.297 		& 0.299 		& 0.303 		& 0.251 		& 0.283 		& 0.295 		& 0.300 		& 0.304 \\
        StyleSwin \\
        \quad 1024$^{\heartsuit}$           & 0.276 		& 0.282 		& 0.284 		& 0.281 		& 0.278 		& 0.251 		& 0.263 		& 0.258 		& 0.255 		& 0.247 \\
        \quad 256$^{\spadesuit}$            & 0.273 		& 0.281 		& 0.285 		& 0.284 		& 0.281 		& 0.256 		& 0.277 		& 0.281 		& 0.277 		& 0.275 \\
        \quad 1024 (\textit{trunc.} $\phi=0.7$)      & 0.276 		& 0.284 		& 0.286 		& 0.290 		& 0.288 		& 0.243 		& 0.263 		& 0.274 		& 0.277 		& 0.275 \\
        \quad 256 (\textit{trunc.} $\phi=0.7$)       & 0.272 		& 0.280 		& 0.281 		& 0.282 		& 0.281 		& 0.248 		& 0.267 		& 0.275 		& 0.274 		& 0.269 \\
        Diffusion-GAN \\
        \quad 1024$^{\heartsuit\spadesuit}$ & 0.278 		& 0.294 		& 0.301 		& 0.303 		& 0.306 		& 0.262 		& 0.288 		& 0.298 		& 0.301 		& 0.302 \\
        \quad 1024 (\textit{trunc.} $\phi=0.7$)      & 0.277 		& 0.291 		& 0.300 		& 0.291 		& 0.308 		& 0.249 		& 0.279 		& 0.289 		& 0.296 		& 0.301 \\
        StyleNeRF \\
        \quad 1024 & 0.268 		& 0.277 		& 0.282 		& 0.285 		& 0.287 		& 0.238 		& 0.252 		& 0.262 		& 0.272 		& 0.281 \\
        \quad 256                           & 0.264 		& 0.272 		& 0.277 		& 0.280 		& 0.283 		& 0.235 		& 0.244 		& 0.252 		& 0.256 		& 0.263 \\
        \quad 1024 (\textit{trunc.} $\phi=0.7$)      & 0.268 		& 0.276 		& 0.281 		& 0.284 		& 0.286 		& 0.233 		& 0.244 		& 0.253 		& 0.261 		& 0.270 \\
        \quad 256 (\textit{trunc.} $\phi=0.7$)       & 0.264 		& 0.271 		& 0.277 		& 0.280 		& 0.282 		& 0.232 		& 0.238 		& 0.246 		& 0.251 		& 0.255 \\
        StyleSDF \\
        \quad 1024$^{\heartsuit\spadesuit}$ & 0.275 		& 0.278 		& 0.273 		& --    		& --    		& 0.253 		& 0.259 		& --    		& --    		& -- \\
        \quad 1024 (\textit{trunc.} $\phi=0.7$)      & 0.273 		& 0.279 		& 0.275 		& --    		& --    		& 0.242 		& 0.252 		& 0.248 		& --    		& -- \\ 
        EG3D \\
        \quad 512$^{\heartsuit\spadesuit}$  & 0.284 		& 0.297 		& 0.303 		& 0.305 		& 0.308 		& 0.257 		& 0.284 		& 0.287 		& 0.287 		& 0.281 \\
        \quad 512 (\textit{trunc.} $\phi=0.7$)       & 0.282 		& 0.295 		& 0.300 		& 0.301 		& 0.305 		& 0.246 		& 0.268 		& 0.276 		& 0.278 		& 0.276 \\
        \bottomrule
    \end{tabular}
    \end{adjustbox}
\end{table}

\section{Societal Impact}
\label{app:societal-impact}

In general, improved generative modeling makes it easier to generate fake media  (e.g., DeepFakes; \citealp{westerlund2019emergence,Vaccari2020DeepfakesAD}) and privacy leaks (e.g., identity-conditioned human face synthesis, information leaks from large-scale pre-training data of text-to-image diffusion models). Additionally, in the particular case of this paper, one could encounter biases image editing as a result of applying \cyclediff to text-to-image diffusion models that reflect the natural biases in large text-image pre-training data. On the other hand, improved generative modeling can bring benefits to synthesis of humans and new ways of human communication in AR/VR. Moreover, we point out that there exist many current research works and tools that can efficiently detect fake media or can manage privacy leaks during pre-training. We encourage researchers and practitioners to consider these risks and remedies when using the methods developed in this paper.

\end{document}